\documentclass[nohyperref]{article}

\usepackage{microtype}
\usepackage{graphicx}
\usepackage{subfigure}
\usepackage{booktabs} % for professional tables
\usepackage[depth=subsection]{bookmark}

\usepackage{hyperref}

\usepackage[accepted]{icml2022}

\usepackage{amsmath}
\usepackage{amssymb}
\usepackage{mathtools}
\usepackage{amsthm}

\usepackage[capitalize,noabbrev]{cleveref}

\theoremstyle{plain}
\newtheorem{theorem}{Theorem}[section]
\newtheorem{proposition}[theorem]{Proposition}
\newtheorem{lemma}[theorem]{Lemma}

\theoremstyle{definition}
\newtheorem{definition}[theorem]{Definition}
\newtheorem{assumption}[theorem]{Assumption}
\theoremstyle{remark}
\newtheorem{remark}[theorem]{Remark}

\usepackage[textsize=tiny]{todonotes}

\newcommand\VECTOR{\mathbf}  % 向量
\newcommand\SPACE{\mathcal}  % 集合
\usepackage{color}

\newcommand{\blue}[1]{\textcolor{blue}{#1}}
\usepackage{bbm}

\icmltitlerunning{Reachability Constrained Reinforcement Learning}

\begin{document}

\twocolumn[
\icmltitle{Reachability Constrained Reinforcement Learning}

% It is OKAY to include author information, even for blind
% submissions: the style file will automatically remove it for you
% unless you've provided the [accepted] option to the icml2022
% package.

% List of affiliations: The first argument should be a (short)
% identifier you will use later to specify author affiliations
% Academic affiliations should list Department, University, City, Region, Country
% Industry affiliations should list Company, City, Region, Country

% You can specify symbols, otherwise they are numbered in order.
% Ideally, you should not use this facility. Affiliations will be numbered
% in order of appearance and this is the preferred way.
\icmlsetsymbol{equal}{*}
\begin{icmlauthorlist}
\icmlauthor{Dongjie Yu}{equal,svm}
\icmlauthor{Haitong Ma}{equal,svm,seas}
\icmlauthor{Shengbo Eben Li}{svm}
\icmlauthor{Jianyu Chen}{iiis,qizhi}

%\icmlauthor{}{sch}
%\icmlauthor{}{sch}
\end{icmlauthorlist}

\icmlaffiliation{svm}{School of Vehicle and Mobility, Tsinghua University, Beijing, China}
\icmlaffiliation{iiis}{Institute for Interdisciplinary Information Sciences, Tsinghua University, Beijing, China}
\icmlaffiliation{qizhi}{Shanghai Qizhi Institute, Shanghai, China}
\icmlaffiliation{seas}{John A. Paulson School of Engineering and Applied Sciences, Harvard University, Cambridge, Massachusetts, USA. This work was conducted during Haitong's graduate study at Tsinghua University.}

\icmlcorrespondingauthor{Shengbo Eben Li}{lishbo@tsinghua.edu.cn}

% You may provide any keywords that you
% find helpful for describing your paper; these are used to populate
% the "keywords" metadata in the PDF but will not be shown in the document
\icmlkeywords{Safe Reinforcement Learning, Constrained Reinforcement Learning, Reinforcement Learning, Reachability Analysis, Feasible Set}

\vskip 0.3in
]

% this must go after the closing bracket ] following \twocolumn[ ...

% This command actually creates the footnote in the first column
% listing the affiliations and the copyright notice.
% The command takes one argument, which is text to display at the start of the footnote.
% The \icmlEqualContribution command is standard text for equal contribution.
% Remove it (just {}) if you do not need this facility.

%\printAffiliationsAndNotice{}  % leave blank if no need to mention equal contribution
\printAffiliationsAndNotice{\icmlEqualContribution} % otherwise use the standard text.

\begin{abstract}
% Constrained Reinforcement Learning (CRL) has gained significant interests recently, since the satisfaction of safety constraints is critical for real world problems. However, existing CRL methods constraining discounted cumulative costs generally lack rigorous definition and guarantee of safety. On the other hand, in the safe control research, safety is defined as persistently satisfying certain state constraints, with rich theoretical and practical techniques for ensuring it. A key characteristic of this setting is that such persistent safety is possible only on a subset of the state space, called feasible set. Moreover, for a certain environment, there exists an optimal largest feasible set. Recent studies incorporating safe control techniques (e.g, energy-based methods) into CRL result in conservative estimation of feasible sets, which in turn harms the performance of learned policies. In this paper, we proposes a reachability constrained reinforcement learning (RCRL) method, which is able to discover the largest feasible set and the optimal policy satisfying persistent safety in this set. We prove that the proposed algorithm converges to a local optimum with the theoretical largest feasible set, using the multi-time scale stochastic approximation theory. Empirical results on different benchmarks such as safe-control-gym and Safety-Gym show that RCRL outperforms state-of-the-art CRL baselines in terms of the learned feasible set, the performance in optimal criteria, and constraint satisfaction.

Constrained reinforcement learning (CRL) has gained significant interest recently, since safety constraints satisfaction is critical for real-world problems. However, existing CRL methods constraining discounted cumulative costs generally lack rigorous definition and guarantee of safety. In contrast, in the safe control research, safety is defined as persistently satisfying certain state constraints. Such persistent safety is possible only on a subset of the state space, called feasible set, where an optimal largest feasible set exists for a given environment. Recent studies incorporate feasible sets into CRL with energy-based methods such as control barrier function (CBF), safety index (SI), and leverage prior conservative estimations of feasible sets, which harms the performance of the learned policy. To deal with this problem, this paper proposes the reachability CRL (RCRL) method by using reachability analysis to establish the novel self-consistency condition and characterize the feasible sets. The feasible sets are represented by the safety value function, which is used as the constraint in CRL. We use the multi-time scale stochastic approximation theory to prove that the proposed algorithm converges to a local optimum, where the largest feasible set can be guaranteed. Empirical results on different benchmarks validate the learned feasible set, the policy performance, and constraint satisfaction of RCRL, compared to CRL and safe control baselines.
\end{abstract}

\section{Introduction}
\label{sec.introduction}

Constrained reinforcement learning (CRL) has gained growing attention due to the safety requirements in the practical applications of RL. The safety specifications in common CRL methods are expected discounted cumulative costs \cite{altman1999cmdp, achiam2017cpo, tessler2018rcpo, Yang2020pcpo}. However, the main deficiency of expected costs is averaging the potential danger at a state to the whole trajectory. For example, the autonomous vehicle should always keep a safe distance to other traffic participants but not keep the cumulative or average distance during a period when the collision might happen at a single time step. Therefore, constraints imposed on expected cumulative costs lack rigorous definition and guarantee of safety.

\begin{figure}
    \centering
    \includegraphics[width=0.55\linewidth]{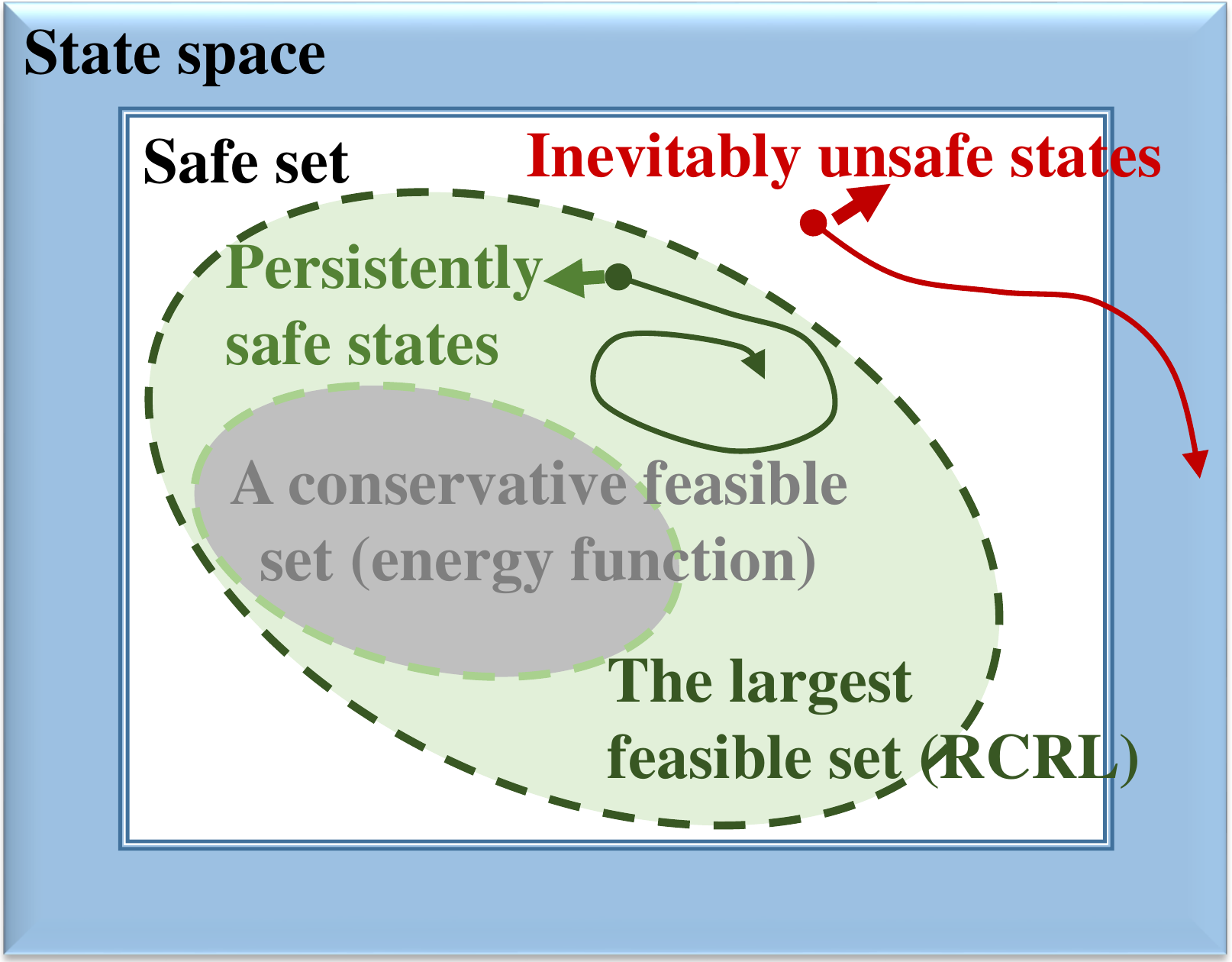}
    % \vspace{-0.1in}
    \caption{The intuitive relationship among the state space, the safe states, feasible sets and the largest feasible set.}
    \vskip -0.1in
    \label{fig:feasible_set_schematic}
\end{figure}

Meanwhile, the persistent state constraint satisfaction in the safe control research clarifies the safety of states with rigorous definitions \cite{liu2014ssa, ames2019cbf, choi2021robust}. Rich theoretical and practical techniques for ensuring safety in such settings are provided, where an important fact is that only a subset of states can be guaranteed safe persistently, called the feasible set. Outside of the feasible set, even temporally safe states will violate the constraints inevitably in the future, no matter what policies to choose, as shown in \cref{fig:feasible_set_schematic}. For example, if a vehicle with high speed is too close to a front obstacle, it is doomed to crash since the deceleration capability is limited. Therefore, accurately identifying feasible sets in CRL can significantly affect the performance and safety of the learned policy.

Some recent RL studies adopt energy-based methods to handle persistent safety and characterize feasible sets. Representatives include control barrier function (CBF) \cite{ames2019cbf, ma2021cbf} and safety index (SI) \cite{liu2014ssa, ma2021joint}. However, these methods rely on prior formulation of the energy function, which results in conservative feasible sets (as \cref{fig:feasible_set_schematic} shows), causing unsatisfying performance sacrifice. Hamilton-Jacobi (HJ) reachability analysis is another branch in the safe control research, which identifies the theoretical largest feasible set \cite{lygeros1999controllers, mitchell2005reachability, bansal2017hj}. Recently, some pioneering studies migrated HJ reachability analysis to model-free RL \cite{fisac2019hjrl, hsu2021reach}. However, these works obtain only the safest policies, leaving the performance criterion (e.g., reward optimization) unconsidered. This safety-only design significantly limits broader applications of HJ reachability analysis in RL.

This paper proposes reachability constrained reinforcement learning (RCRL), which learns the optimal safe policy satisfying persistent safety within the identified largest feasible set.
We leverage reachability analysis to establish the novel self-consistency condition and characterize the feasible sets.
The feasible sets are represented by the safety value function. Intuitively, the function describes the worst constraint-violation in the long term, and its sub-zero level set is the feasible set.
% RCRL introduces a novel reachability constraint imposed on the safety value function, which is a critical notion in HJ reachability analysis. Intuitively, it describes the worst constraint-violation in the long term, and its sub-zero level set is the feasible set. We derive a self-consistency condition of the safety value function and incorporate it into the CRL framework. 
We use the multi-time scale stochastic approximation theory \cite{borkar2009stochastic, chow2017cvar} to prove that the proposed algorithm converges to a local optimum. Empirical results on low-dimensional problems validate the correctness of the learned feasible sets. Further experiments conducted on complex benchmarks such as safe-control-gym \cite{yuan2021scg} and Safety-Gym \cite{Achiam2019safetygym} indicate that RCRL achieves competitive performance while maintaining constraint-satisfaction. Our main contributions are: 
\begin{itemize}
\setlength{\itemsep}{0pt}
\setlength{\parsep}{0pt}
\setlength{\parskip}{0pt}
    \item We are the first to introduce reachability constraints into CRL, which is critical for learning a nearly optimal and persistently safe policy upon its corresponding feasible set. Compared to other feasible set characterization methods, RCRL enlarges the feasible sets and reduces the policy conservativeness.
    % The feasible set can be identified by temporal-difference (TD) learning with the established self-consistency condition, avoiding the time-consuming calculation of the safest policy faced by conventional HJ approaches.
    \item We use the multi-time scale stochastic approximation theory to prove that RCRL converges to a locally optimal policy, which also persistently satisfies the state constraints across the entire largest feasible set if the initialization of states is general.
    \item Comprehensive experiments demonstrate that the proposed RCRL method outperforms CRL and safe control baselines in terms of final performance and constraint satisfaction.
\end{itemize}

\section{Related Work}
\label{sec.related_work}
% \red{\textbf{Constrained Reinforcement Learning (CRL)} is an active research area in safe RL aiming to learn effective and safe policies for sequential decision-making tasks.} \commenthaitong{No need to mention again.} 
\textbf{Constrained reinforcement learning (CRL)} problems are usually formulated as constrained Markov decision process (CMDP) \cite{altman1999cmdp, brunke2021review}. Constrained optimization approaches are adopted to solve CRL problems: (1) penalty function \cite{guan2021idc}; (2) Lagrangian methods 
\cite{tessler2018rcpo, chow2017cvar, duan2021adp, ma2021fac}; (3) trust-region methods \cite{achiam2017cpo, Yang2020pcpo} and (4) other approaches such as conservative updates \cite{bharadhwaj2021conservative}. CMDP relies on the expected discounted cumulative costs and a hand-crafted threshold to improve the safety of policies. However, a proper threshold relies on engineering intuitions and varies in different tasks \cite{qin2021dcrl}. 
% In other words, the crucial feasible sets cannot be defined rigorously and generally.

\textbf{Characterizing feasible sets} is a critical and open problem in safe control research \cite{brunke2021review}. Feasible sets, also called recoverable sets \cite{thomas2021imagining}, are usually represented by safety certificates. Representative safety certificates include energy functions such as CBF \cite{ma2021cbf, choi2021robust, luo2021Learning} and SI \cite{liu2014ssa}. The core idea of the energy function is that the energy of a dynamical system dissipates when it is approaching the safer region \cite{ames2019cbf}. Nevertheless, energy-based methods suffer from conservative or inaccurate feasible sets \cite{ma2021joint}. HJ reachability analysis is a promising way towards general and rigorous derivation of feasible sets \cite{lygeros1999controllers, mitchell2005reachability}. However, it is quite difficult to obtain the largest feasible set because it is represented by a non-trivial partial differentiable equation, whose analytical solution is often intractable \cite{bansal2017hj}. Machine learning, especially RL approaches, is adopted to deal with this problem \cite{fisac2019hjrl, bansal2021deepreach, hsu2021reach}. However, most of the existing reachability studies only care about safety while ignoring other metrics, especially the optimality criterion, limiting reachability analysis approaches from broader applications. \citet{Thananjeyan2021recovery} pretrain a feasible set indicator for switching between the optimal policy and a back-up safe controller during training. A recent CRL study utilizes reachability analysis to learn a purely safe back-up policy \cite{chen2021ego}. Then the agent switches between the safe and optimal policies when interacting with the environment. Different from these switching-based methods, we only learn one policy tackling safety and optimality simultaneously.
% Therefore, to the best of our knowledge, this is the first work to characterize the feasible sets in CRL and consider optimality criterion in reachability methods.

\section{When RL Meets Feasible Sets}
\subsection{Notation}
\label{sec.notation}
We formulate the CRL problem as an MDP with a deterministic dynamic (a reasonable assumption in safe control problems), defined by the tuple $\langle \SPACE{S}, \SPACE{A}, P, r, h, c, \gamma\rangle$ where (1) the state space $\SPACE{S}$ and the action space $\SPACE{A}$ are bounded (possibly continuous); (2) unknown transition probability $P: \SPACE{S}\times\SPACE{A}\times\SPACE{S} \mapsto \{0,1\}$ represents the dynamics; (3) $r: \SPACE{S}\times\SPACE{A} \mapsto \mathbb{R}$ is the reward function; (4) $h: \SPACE{S}\mapsto\mathbb{R}$ is the state constraint. % , where the safe set is defined by $\{s|h(s)\le0\}$.
$c$ is called the cost signal where $c(s)=\mathbbm{1}_{h(s)>0}$, indicating we get $1$ if $h(s)\le0$ is violated and otherwise $0$. (5) $\gamma\in(0,1)$ is the discount factor. A deterministic policy $\pi: \SPACE{S}\mapsto\SPACE{A}$ chooses action $a_t$ at state $s_t$ at time $t$. The initial state distribution is denoted as $d_0(s)$ while $d_{\pi}(s,a)$ is the state-action marginals following $\pi$. We denote the initial state set as $\SPACE{S}_0\triangleq\{s\mid d_0(s)>0\}$.

The objective of standard RL is to find a policy maximizing the expected return (discounted cumulative rewards) $\mathcal{J}(\pi) = \mathbb{E}_{s_t,a_t \sim d_{\pi}} \sum_{t}[\gamma^t r(s_t,a_t)]$. A value function $V^{\pi}(s)\triangleq \mathbb{E}_{s_t,a_t \sim d_{\pi}} \sum_{t}[\gamma^t r(s_t,a_t)|s_0=s]$ represents the potential return in the future from state $s$, satisfying $V^{\pi}(s)=r(s,\pi(s))+\gamma\mathbb{E}_{s' \sim P}[V^{\pi}(s')]$. One can easily find that $\mathcal{J}(\pi)=\mathbb{E}_{s\sim d_0(s)}[V^{\pi}(s)]$.
In CRL, one can define discounted cumulative costs as cost return $\mathcal{J}_c(\pi) = \mathbb{E}_{s_t,a_t \sim d_{\pi}} \sum_{t}[\gamma^t c(s_t)]$ and cost value function $V_c^\pi(s)=c(s)+\gamma \mathbb{E}_{s'\sim P}[V^{\pi}(s')]$ similarly.

We need a few extended notations beyond standard ones. We denote $(s_{t}^\pi\mid s_0=s,\pi), {t\in\mathbb{N}}$ as the state trajectory $\{s_0^\pi,s_1^\pi,\cdots\mid s_0=s,\pi\}$ induced by $\pi$ from $s_0=s$. Let $h(s_{t}^{\pi}\mid s_0=s), t\in\mathbb{N}$ specify the state constraint sequence of the trajectory $\{h(s_0^\pi),h(s_1^\pi),h(s_2^\pi),\cdots\mid s_0=s,\pi\}$. We also denote $h(s_{t}^{\pi}\mid s_0=s)\le0,t\in\mathbb{N}$ as persistently satisfying the constraint, i.e. $h(s_t^\pi)\le0,\forall t \in \mathbb{N}$.

\subsection{Definition of Feasible Sets}
Generally speaking, a state is considered \emph{safe} if it satisfies the state constraint $h(s)\le0$, such as keeping distance from obstacles. The safe set is defined by the set of all safe states:
\begin{definition}[Safe set]
\begin{equation}
\nonumber
\SPACE{S}_c \triangleq \{s\mid h(s)\le0\}.
\end{equation}
% \hfill$\blacksquare$
\end{definition}
However, as stated in \cref{sec.introduction}, some safe states would go dangerous no matter what policy we choose, such as a high-speed vehicle close to a front obstacle. Therefore, what really matters for meaningful safety guarantee is not the \emph{temporary} safety but the \emph{persistent} safety, i.e., $h(s_{t}^{\pi}\mid s_0=s)\le0,t\in\mathbb{N}$. Otherwise, the system will be dangerous sooner or later. In other words, we need to characterize those states starting from which the policy $\pi$ is able to keep the system constraint-satisfactory. We define the feasible set as the set of all the states which are able to be safe persistently.
% \commenthaitong{What policy can be used to define $S$ here, any policy or only feasible policy? if feasible policy, the definition of feasible policy should be introduced.}
\begin{definition}[Feasible set]
\label{def.feasible_set}
The feasible set of \emph{a specific} policy $\pi$ can be defined as
\begin{equation}
\nonumber
\mathcal{S}_f^{\pi} \triangleq \{s \in \SPACE{S} \mid h(s_{t}^{\pi}\mid s_0=s)\le 0, t\in\mathbb{N} \}.
\end{equation}
A policy $\pi$ is feasible if $\mathcal{S}_f^{\pi}\ne\emptyset$ and otherwise it is infeasible.
% \commenthaitong{Notions about what $h(s_{t}^{\pi}\mid s_0=s)$ means are not defined}
The largest feasible set $\mathcal{S}_f$ is a subset of $\SPACE{S}$ composed of states from which there exists \emph{at least one} policy that keeps the system satisfying the constraint, i.e.,
\begin{equation}
\nonumber
\mathcal{S}_f \triangleq \{s \in \SPACE{S} \mid \exists \pi, h(s_{t}^{\pi}\mid s_0=s)\le 0, t\in\mathbb{N} \}.
\end{equation}
% \hfill$\blacksquare$
\end{definition}

% \red{To measure the feasibility of states quantitatively, a safety value function $V^{\pi}_{h}(s)$ can be defined according to \cref{def.feasible_set},}
% \commenthaitong{``measure the feasibility of states quantitatively'' might be too intuitive, explain why we use the max, like expanding the worst constraint violation in the following.}
To guarantee that all the states in the trajectory $\{s_{t}^\pi\mid s_0=s,\pi\}, {t\in\mathbb{N}}$ are safe, we only need to guarantee that the worst-case i.e., the \emph{maximum} violation in the trajectory is below zero, which brings the following definition:
\begin{definition}[Safety value function]
\label{def.safety_value_func}
\begin{equation}
\label{eq.def_feasibility}
    V^{\pi}_{h}(s) \triangleq \max_{t\in \mathbb{N}} h(s_{t}^{\pi}\mid s_0=s),
\end{equation}
is the worst constraint violation in the long term. % \hfill$\blacksquare$
\end{definition}

The safety value of a given state $s$ varies when the policy $\pi$ changes. The best value we can get is the one where we choose the policy minimizing the constraint violation and we call it the optimal safety value function:
\begin{equation}
\nonumber
\label{eq.optimal_feasibility}
    V^{*}_{h}(s) \triangleq \min_{\pi}\max_{t\in \mathbb{N}} h(s_{t}^{\pi}\mid s_0=s).
\end{equation}
% HJ reachability analysis focuses on finding the \red{optimal safety value function}\todo{why we need optimal safety value functions? I guess the reason is taht we want to find the largest feasible set. Better to explain more about hte relationship between optimal safety value function and largest feasible set.}, which can be obtained from the definition of the largest feasible set:
One can easily observe that the (largest) feasible set is the sub-zero level set of the (optimal) safety value function, i.e.,
\begin{equation}
\nonumber
\SPACE{S}_f^\pi = \{s \mid V_{h}^{\pi}(s) \le 0\}, \quad
\SPACE{S}_f = \{s \mid V_{h}^{*}(s) \le 0\}.
\end{equation}
and clearly $\SPACE{S}_f^\pi \subseteq \SPACE{S}_f$ for $\forall \pi$.

Safety problems in reality are more about the worst-case through time other than the cumulative or average costs \cite{fisac2019hjrl}, where the latter is often the case in previous CRL. The safety value function measures the safety of the most dangerous state on the trajectory generated by $\pi$. Specifically, if $V^{\pi}_{h}(s) \le 0$, the most dangerous state is safe, so the safety of the system can be guaranteed. Otherwise, the policy $\pi$ would definitely cause state constraint violations in the future. Therefore, once $V^{\pi}_{h}(s) \le 0$, which we call that the \textit{reachability constraint} is fulfilled, the agent is guaranteed to be inside the feasible set since the state constraint could be satisfied persistently.

% \red{However, when $V^{*}_{h}(s) > 0$, the system is doomed to violate the constraint because there exists at least one $t\in \mathbb{N}$ that $h(s_{t})\ge0$ no matter what policy we choose. In other words, only states inside the largest feasible set is really safe and meaningful for safe control problems so we need to characterize them through its safety value function.}\todo{overlapped with the statement before definition of safety value function. Better to simplify these intuitive explanations.}

% The \red{relationship of the safety value function of a state between that of its successive state } \commenthaitong{Better to use this: `` ...''}

\subsection{Computation of the Safety Value Function}
Although the safety value function does not use the discounted cumulative formulation like the common value functions in RL, we can still use the temporal difference learning technique to get it. \citet{fisac2019hjrl} proposes the following lemma about the optimal safety value function:
\begin{lemma}[Safety Bellman equation (SBE)]
\label{theo.sbe}
\begin{equation}
\label{eq.sbe}
V^{*}_{h}(s) = \max\left\{h(s), \min_{a\in\SPACE{A}} V^{*}_{h}(s')\right\}
\end{equation}
holds for $\forall s \in \SPACE{S}$, where $s'$ is the successive state of state $s$.
\end{lemma}
We extend \cref{theo.sbe} to a general form applicable to any policy, which is called the self-consistency condition:
\begin{theorem}[Self-consistency condition of the safety value function]
\label{theo.consistency}
\begin{equation}
\label{eq.self-con-condition}
V^{\pi}_{h}(s) = \max\left\{h(s), V^{\pi}_{h}(s')\right\}
\end{equation}
holds for $\forall s \in \SPACE{S}$ and $\forall \pi$, where $s'$ is the successive state at state $s$ following $\pi$.
\end{theorem}
\vspace{-0.15in}
\begin{proof}
See \cref{proof.sbe}.
\end{proof}

\begin{remark}[Optimality]
The largest feasible set can be obtained by solving SBE \eqref{eq.sbe} for the optimal safety value function. However, this will lead to a policy always pursuing the lowest constraint violation, i.e., a purely safe policy. The purely safe policy does not tackle the optimality specifications. For example, a robotic arm should catch the objects as \emph{quickly} as possible with only \emph{bounded} torques. We do not need to choose the safest action at states which are interior points of the feasible set (non-safety-critical) \cite{salar2021projectionrl}. Intuitively, the safest action has to be taken only at states on the boundary of the feasible set (safety-critical states).
To address this issue, some studies design the switching rules between the purely safe policy and optimal policy \cite{chen2021ego}. In contrast, this paper learns only one policy tackling safety and optimality simultaneously. \cref{theo.consistency} could compute the safety value for this unified policy.
% Therefore, if we optimize the expected return at non-safety-critical states and guarantee safety at safety-critical states, we are able to consider RL performance and safety comprehensively, where \cref{theo.consistency} will help \blue{by offering a bootstrap manner to compute the safety value function for a \textit{certain} policy.}
\end{remark}

\begin{remark}[Scalability]
\label{remark.scalability}
Besides characterizing the persistent safety of agents, reachability constraints are also scalable to constraints on cumulative quantities (i.e., safe budget) similar to conventional CRL, enabling it to be a general constraint formulation. A safe budget for the whole trajectory can be seen as the remaining budget constraint on any state during the trajectory. Let the budget constraint be $\sum_t{c(s_{\rm t}|s)}-\eta(s)\le0$, where $c(s_t)$ is the consumption of one step and $\eta(s)$ is the remaining budget at state $s$. We define $h(s)=-\eta(s)$ and thus $V_h^\pi(s)$ equals the worst-case over-budget whose being greater than 0 is unacceptable.
\end{remark}

% The proposed self-consistency condition does not limit the policy to be only safety-oriented and we can integrate it into the framework of CRL, learning the safety value function and feasible set when maximizing the expected return at the same time.

\section{Reachability Constrained Reinforcement Learning}
% \blue{In this section, we start from the original CRL problem, and introduce pointwise constraints to CRL, which is necessary for safety-critical systems. After proving the equivalence between pointwise constraints and constraints on the safety value function, we formally propose the problem statement of RCRL, followed by its benefits over conventional CMDP-based CRL methods and HJ reachability analysis methods. Under some proper settings, the feasible set of the solution to RCRL problem equals to the largest one defined in \cref{def.feasible_set}. Then a general RCRL algorithm with convergence guarantee will be devised.}

In this section, we formally propose the novel RCRL problem which guarantees the persistent safety of the policy. Furthermore, if the initialization of the state covers the largest feasible set, the feasible set solved by RCRL equals the largest one defined in \cref{def.feasible_set}. Then the RCRL algorithm with a convergence guarantee will be devised.

\subsection{Problem Statement}
% RCRL, why it is persistently safe
% when the largest, what is the condition
% why good
Given an MDP defined in \cref{sec.notation} and an initial state distribution $d_0$, RCRL aims to find the optimal policy $\pi^*$ to the following optimization problem:
\begin{equation}
\tag{RCRL}
\begin{aligned}
\label{eq.rcrl_problem}
\max_{\pi} &\quad\mathbb{E}_{s\sim d_0(s)} [V^{\pi}(s) \cdot \mathbbm{1}_{s\in \SPACE{S}_f} - V_h^{\pi}(s) \cdot \mathbbm{1}_{s\notin \SPACE{S}_f}] \\
\mbox{subject to} &\quad  V^{\pi}_h(s) \le 0, \forall s \in \SPACE{S}_f \cap \SPACE{S}_0,
\end{aligned}
\end{equation}
where $\mathbbm{1}_{A}=1$ holds when the event $A$ is \texttt{true} and otherwise $\mathbbm{1}_{A}=0$. Intuitively, \textbf{for initial states inside the largest feasible set}, i.e. $s\in \SPACE{S}_0\cap\SPACE{S}_f$, \eqref{eq.rcrl_problem} aims to maximize the expected return and ensure the persistent safety when following this policy. However, \textbf{for initial states outside the largest feasible set}, i.e., $s\in \SPACE{S}_0\setminus\SPACE{S}_f \triangleq \{ s \mid s \in \SPACE{S}_0, s \notin \SPACE{S}_f\}$, the state constraint $h(s)\le0$ will be violated sooner or later and it is impossible to satisfy the reachability constraint. Thus, it is meaningless to optimize the return of these infeasible states, and we only try to find the safest actions by minimizing the safety value functions.

The formulation in \eqref{eq.rcrl_problem} is different from the common CRL formulations in \cite{achiam2017cpo, tessler2018rcpo, Yang2020pcpo} where the constraint is imposed on the expectation of cost return:
\begin{equation}
\begin{aligned}
\label{eq.crl_problem}
\max_{\pi} & \quad \mathbb{E}_{s\sim d_0(s)} [V^{\pi}(s)] \\
\mbox{subject to} & \quad  \mathbb{E}_{s\sim d_0(s)}[V^{\pi}_c(s)] \le \eta,
\end{aligned}
\end{equation}
where $\eta$ is the cost threshold, $V_c^\pi$ is the cost value function. However, as stated in Section \ref{sec.related_work}, choosing $\eta$ is tricky
and it is hard to migrate the expectation-based CRL formulation to CRL problems with state constraints, such as \eqref{eq.rcrl_problem}.

Specially, if the the initial states cover the largest feasible set, the solution to \eqref{eq.rcrl_problem} also has the largest feasible set, which is given by the following proposition:
\begin{proposition}[The largest feasible set]
\label{theo.equivalent_region}
Assuming $\SPACE{S}_0 \cap \SPACE{S}_f \ne \emptyset$, for any feasible $\pi$ of problem \eqref{eq.rcrl_problem}, we have $\SPACE{S}_f^{\pi}=\SPACE{S}_f$ if $\SPACE{S}_f \subseteq \SPACE{S}_0$.
\end{proposition}
\begin{proof}
See \cref{proof.region}.
\end{proof}
% \red{\cref{theo.equivalent_region} reveals that though the constraint in \eqref{eq.rcrl_problem} is imposed on the safety value function of a specific policy rather than the optimal safety value function, we can still capture the largest feasible set without computing the safest policy if the largest feasible set is covered by the initial states distribution, giving freedom to the policy in RCRL.}
% \commenthaitong{This paragraph might be redundant since the next paragraph reemphasize the advantages.}

Overall, \eqref{eq.rcrl_problem} has three significant advantages: (1) Compared to conventional CRL approaches, \eqref{eq.rcrl_problem} considers the vital persistent safety of the system because every single time step in the future is guaranteed to be safe when reachability constraints are satisfied. (2) Compared to HJ reachability analysis studies, RCRL considers performance optimality besides safety. (3) Compared to other RL methods with feasible sets, any feasible policy of \eqref{eq.rcrl_problem} renders the largest feasible set if the initial states cover the feasible sets. The largest feasible sets brings less conservativeness and better performance because the policy could drive the system towards states with higher return.

\subsection{Lagrangian-based Algorithm with Statewise Constraints}
\label{sec.algo}
% \commentdongjie{Change 'Multiplier Network' to 'Statewise Multiplier' for generality}
We leverage the Lagrange multiplier method to solve problem \eqref{eq.rcrl_problem}, which is a common approach to CRL \cite{chow2017cvar,Achiam2019safetygym,tessler2018rcpo}. The key idea of Lagrangian-based methods is to descend in $\pi$ and ascend in the multiplier $\lambda$ using the gradients of the Lagrangian $\mathcal{L}(\pi, \lambda)$ w.r.t. $\pi$ and $\lambda$ and to finally reach the optimal policy which satisfies the constraints.
Notably, constraints in \eqref{eq.rcrl_problem} are imposed on \emph{each state} in $\SPACE{S}_0\cap\SPACE{S}_f$, which is significantly different from typical ones on only the expectation among the states in \cite{achiam2017cpo, chow2017cvar, tessler2018rcpo}. We call these type of constraints as \emph{statewise constraints}. In this case, the multiplier is not longer a scalar but a vector (infinite-dimension in infinite states cases). Some preliminary studies discuss statewise constraints on the density of the state distribution \cite{chen2019duality,qin2021dcrl}. A more general Lagrangian-based solution for statewise constraints with \emph{an approximation to the multipliers} is discussed recently in \cite{ma2021fac,ma2021joint}. Without loss of generality, we denote the statewise multiplier as a function $\lambda: \SPACE{S}\mapsto[0,+\infty) \cup \{+\infty\}$. Then the Lagrangian of \eqref{eq.rcrl_problem} can be formulated as:
\begin{equation}
\begin{aligned}
\label{eq.lagrangian}
\mathcal{L}(\pi, \lambda) = & \mathbb{E}_{s\sim d_0}\left[-V^\pi(s)\cdot\mathbbm{1}_{s\in \SPACE{S}_f} + V_h^\pi(s)\cdot\mathbbm{1}_{s\notin \SPACE{S}_f}\right] \\
&+ \int_{\SPACE{S}_f \cap \SPACE{S}_0} {\lambda(s){V^\pi_h(s)} {\rm d} s}
\end{aligned}
\end{equation}
% the original lagrangian
% derivation of the surrogate lagrangian
%
% only the initial distribution is accessible, we cannot get S_f so the lagrangian is intractable.
% a surrogate approach
% a common choice initialize the state in the safe set, which is a set containing the largest feasible set.
% however, if states outside, infeasible. 
% multiplier goes to infty, the lagrangian 
% to avoid the divergence, set an upper bound, prove the equivalence between the surrogate and the lagrangian when the ub goes to infty
The most significant issue when we are tackling \eqref{eq.lagrangian} is that we cannot obtain the largest feasible set $\SPACE{S}_f$ in advance, which means $\SPACE{S}_f \cap \SPACE{S}_0$ is unknown. However, the initial distribution $d_0$ is usually accessible and we propose a surrogate Lagrangian in the form of expectation w.r.t. $d_0$ instead:
\begin{equation}
\label{eq.surrogate_lag}
\hat{\mathcal{L}}(\pi, \lambda) = \mathbb{E}_{s\sim d_0}[-V^\pi(s) + \lambda(s){V^\pi_h(s)}]
\end{equation}
A common choice of the initial distribution is the uniform distribution in the safe set $\SPACE{S}_c$ which covers the largest feasible set. However, it is inevitable that there are initial states outside $\SPACE{S}_f$. The constraint $V_h^{\pi}(s)\le0$ can never be satisfied for those states outside $\SPACE{S}_f$. Thus, each corresponding multiplier $\lambda(s)$ will go to $+\infty$ because $\lambda(s)$ tries to maximize the product of itself and a positive scalar $V_h^\pi(s)$, resulting in the divergence of \eqref{eq.surrogate_lag}. We can set a large upper bound $\lambda_{\rm max}$ to the statewise multiplier to avoid the divergence. We claim that when $\lambda_{\rm max}\rightarrow + \infty$, solving the surrogate Lagrangian is equivalent to solving the original one:
\begin{proposition}[Equivalent Lagrangian]
\label{theo.equivalent_lag}
Assume both \eqref{eq.lagrangian} and \eqref{eq.surrogate_lag} have the unique optimal solution. If we denote $\pi^*=\arg\min_{\pi}\max_{\lambda}{\mathcal{L}}$, $\hat{\pi}^*=\arg\min_{\pi}\max_{\lambda}\hat{\mathcal{L}}$, we have $\lim_{\lambda_{\rm max}\to+\infty}\hat{\pi}^*=\pi^*$.
\end{proposition}
\begin{proof}
See \cref{proof.equivalent_lag}.
\end{proof}
% \vspace{-0.2in}

\begin{remark}
The surrogate Lagrangian can be regarded as an expected weighted sum of $-V^\pi(s)$ and $V^\pi_h(s)$ using weights $\lambda(s)$. We separate the surrogate Lagrangian into two parts here, the expectations on feasible and infeasible initial states. For those feasible initial states, $\lambda(s)$ is finite and the surrogate Lagrangian calculates the expected weighted sums of $-V^\pi(s)$ and $V^\pi_h(s)$. For those infeasible initial states, $\lambda(s)\to +\infty$, so $-V^\pi(s)$ is ignored in the weighted sums and the surrogate Lagrangian is only dominated by the expected $V^\pi_h(s)$.
\end{remark}

Hence, solving problem \eqref{eq.rcrl_problem} can be approximated by finding the saddle point of the surrogate \eqref{eq.surrogate_lag}:
\begin{equation}
\label{eq.rcrl_lag}
\min_{\pi}\max_{\lambda}\hat{\mathcal{L}}(\pi, \lambda).
\end{equation}

% \commenthaitong{need some linking words}
Consider the common actor-critic framework with state-action value functions. We have the state-action (safety) value $Q(s,\pi(s))=V^\pi(s)$ and $Q_h(s,\pi(s))=V_h^\pi(s)$ due to the deterministic dynamics and policy. We also adopt parameterized Q-function $Q(s,a;\omega)$, safety Q-value function $Q_h(s,a;\phi)$, a policy $\pi(s;\theta)$, and the statewise multiplier $\lambda(s;\xi)$. The Lagrangian thus becomes $\hat{\mathcal{L}}(\theta, \xi)$. Note that sometimes we use $Q_\omega, \pi_\theta, \lambda_\xi$ for short. Now we derive the objectives of the parameterized functions.

\begin{algorithm}[tb]
   \caption{Template for actor-critic RCRL}
   \label{alg.rcrl}
\begin{algorithmic}
   \STATE {\bfseries Input:} MDP $M$ with constraint $h(\cdot)$, critic and safety value function learning rate $\beta_1(k)$,  actor learning rate $\beta_2(k)$, multiplier learning rate $\beta_3(k)$
   \STATE {\bfseries Initialization:} q-function parameters $\omega=\omega_0$, safety q-function parameters $\phi=\phi_0$, policy parameters $\theta=\theta_0$, multiplier parameters $\xi=\xi_0$
   \FOR{$k=0,1,\ldots$}
   \STATE Initialize state $s_0\sim d_0$.
   \FOR{$t=0$ {\bfseries to} $T-1$}
   \STATE Select action $a_t=\pi_\theta(s_t)$, observe next state $s_{t+1}$, reward $r_t$ and constraint $h(s_t)$
   \STATE {\bfseries Critic update} $\omega_{k+1}=\omega_k-\beta_1(k)\hat{\nabla}_{\omega}{\mathcal{J}_Q}(\omega)$
   \STATE {\bfseries Safety value update}\\\qquad $\phi_{k+1}=\phi_k-\beta_1(k)\hat{\nabla}_{\phi}{\mathcal{J}_{Q_h}}(\phi)$
   \STATE {\bfseries Actor update} $\theta_{k+1}=\Gamma_\Theta\left(\theta_k-\beta_2(k)\hat{\nabla}_{\theta}{\mathcal{J}_{\pi}}(\theta)\right)$
   \STATE {\bfseries Multiplier update}\\\qquad $\xi_{k+1}=\Gamma_\Xi\left(\xi_k+\beta_3(k)\hat{\nabla}_{\xi}{\mathcal{J}_{\lambda}}(\xi)\right)$
   \ENDFOR
   \ENDFOR
   \STATE {\bfseries return} parameters $\omega, \phi, \theta, \xi$
\end{algorithmic}
\end{algorithm}

The Q-value function update is the standard one in popular RL \cite{sutton2018rl}, which can be seen in \cref{sec.grad_derivation}. The safety Q-value function of the current policy is updated according to the self-consistency condition in \cref{theo.consistency}, i.e., minimizing the mean squared error:
\begin{equation}
\label{eq.safety_q_loss}
\mathcal{J}_{Q_h}(\phi) = \mathbb{E}_{s \sim \mathcal{D}} \left[ 1/2\left(Q_h(s,a;\phi)-\hat{Q}_h(s,a)\right)^2 \right],
\end{equation}
where
\begin{equation}
\begin{aligned}
% \label{eq.q_target}
\label{eq.q_h_target}
\hat{Q}_h(s,a)&=(1-\gamma)h(s) \\
&+ \gamma \mathbb{E}_{s'\sim P}[\max \{h(s), Q_h(s',\pi(s');\phi)\}],
\end{aligned}
\end{equation}
$\mathcal{D}$ is the distribution of previously sampled states and actions (i.e., $d_{\pi}$), or a replay buffer, and $a$ is the action taken at $s$. Note that the discounted version of self-consistency condition is for convergence in \cref{proof.convergence}, as in \cite{fisac2019hjrl}.

As aforementioned, the purpose of policy $\pi_\theta$ is to descend the Lagrangian while the multiplier tries to ascend it:
\begin{equation}
\label{eq.policy_loss}
\begin{aligned}
&\mathcal{J}_{\pi}(\theta) = \mathcal{J}_{\lambda}(\xi) \\
&= \mathbb{E}_{s\sim \mathcal{D}}[-Q(s,\pi_\theta(s);\omega) + \lambda_\xi(s){Q_h(s,\pi_\theta(s);\phi)}].
\end{aligned}
\end{equation}

\cref{alg.rcrl} provides the pseudo-code of an actor-critic version of RCRL. The algorithm alternates between interacting with the environment and updating the parameter vectors with stochastic gradients $\hat{\nabla}_{\omega}{\mathcal{J}_Q}(\omega)$, $\hat{\nabla}_{\theta}{\mathcal{J}_{\pi}}(\theta)$, $\hat{\nabla}_{\theta}{\mathcal{J}_{\pi}}(\theta)$, and $\hat{\nabla}_{\xi}{\mathcal{J}_{\lambda}}(\xi)$, whose derivation can be seen in \cref{sec.grad_derivation}. In the algorithm, the $\Gamma_\Psi(\psi)$ operator projects a vector $\psi\in\mathbb{R}^\kappa$ to the closet point in a compact and convex set $\Psi\subseteq\mathbb{R}^\kappa$, i.e., $\Gamma_\Psi(\psi)=\arg\min_{\hat{\psi}\in\Psi}\Vert \hat{\psi}-\psi\Vert ^2$ where $\psi$ is denoted as any one of $\theta, \xi$. These projection operators are necessary for the convergence of the actor-critic algorithm \cite{chow2017cvar}. A policy-gradient version of RCRL is designed similarly in \cref{alg:rco}.

\section{Convergence Analysis}
% \subsection{Main Convergence Results}
Under moderate assumptions, we can provide a convergence guarantee of \cref{alg.rcrl}. The convergence analysis follows heavily from the convergence proof of multi-time scale stochastic approximation algorithms \cite{chow2017cvar}. We also utilize theorems of combining the ODE (ordinary differential equation) viewpoint and stochastic approximation from \cite{borkar2009stochastic}. We first introduce the necessary assumptions.
\begin{assumption}[Finite MDP]
\label{assume.mdp}
The MDP is finite (finite state and action space, i.e., $\vert\SPACE{S}\vert<\infty, \vert\SPACE{A}\vert<\infty$), and $\SPACE{S}$ and $\SPACE{A}$ are both bounded. The first-hitting time of the MDP $T_{\pi, s}$ is bounded almost surely over all policy $\pi$ and all initial states $s\in\SPACE{S}_0$. We refer the upper bound as $T$. The reward function and constraint value of a single step are bounded by $r_{\rm{max}}$ and $h_{\rm max}$, respectively. Hence, the value function is upper bounded by $r_{\rm{max}}/(1-\gamma)$.
\end{assumption}

\begin{assumption}[Strict Feasibility]
\label{assume.feasibility}
There exists a policy $\pi(\cdot;\theta)$ such that $V^{\pi_\theta}_h(s)\le0,\forall s\in\SPACE{S}_0\cap\SPACE{S}_f\ne\emptyset$.
\end{assumption}

\begin{assumption}[Differentiability]
\label{assume.differentiability}
For any state-action pair $(s,a)$, $Q(s,a;\omega)$ and $Q_h(s,a;\phi)$ are continuously differentiable in $\omega$ and $\phi$, respectively. Moreover, $\nabla_a{Q(s,a;\omega)}$ and $\nabla_a{Q_h(s,a;\phi)}$ are Lipschitz functions in $a$, for $\forall s \in \SPACE{S}, \forall \omega \in \Omega$, and $\forall \phi \in \Phi$. For $\forall s\in\SPACE{S}, \forall a\in\SPACE{A}$, $\nabla_a{Q(s,a;\omega)}$ is a Lipschitz function in $\omega$ and $\nabla_a{Q_h(s,a;\phi)}$ is a Lipschitz function in $\phi$. For any state $s$, $\pi(s;\theta)$ is continuously differentiable in $\theta$ and $\nabla_\theta{\pi(s;\theta)}$ is a Lipschitz function in $\theta$. For any state $s$, $\lambda(s;\xi)$ is continuously differentiable in $\xi$ and $\nabla_\xi{\lambda(s;\xi)}$ is a Lipschitz function in $\xi$.
\end{assumption}

\begin{assumption}[Step Sizes]
\label{assume.step_size}
The step size schedules $\{\beta_1(k)\}, \{\beta_2(k)\}$, and $\{\beta_3(k)\}$ satisfy
\begin{equation}
\begin{aligned}
\nonumber
& \sum_k{\beta_1(k)}=\sum_k{\beta_2(k)}=\sum_k{\beta_3(k)}=\infty \\
& \sum_k{\beta_1(k)^2}, \sum_k{\beta_2(k)^2}, \sum_k{\beta_3(k)^2} < \infty \\
& \beta_3(k)=o\left(\beta_2(k)\right), \beta_2(k)=o\left(\beta_1(k)\right).
\end{aligned}
\end{equation}
\end{assumption}
These step-size schedules satisfy the standard conditions for stochastic approximation algorithms, and ensure that the critic update is on the fastest time scale $\{\beta_1(k)\}$, the policy update is on the intermediate time scale $\{\beta_2(k)\}$, and the multiplier is on the slowest one $\{\beta_3(k)\}$. 
Now we come to the position where the convergence of actor-critic RCRL can be provided.
\begin{theorem}
\label{theo.convergence}
Under \cref{assume.mdp} to \ref{assume.step_size}, the policy sequence updated in \cref{alg.rcrl} converges almost surely to a locally optimal policy for the reachability constrained policy optimization problem \eqref{eq.rcrl_problem}.
\end{theorem}
\begin{proof}
See \cref{proof.convergence}.
\end{proof}

% \subsection{Discussions}
The conditions for convergence may be strict and ideal such that we have to make some simplification and approximation to make the RCRL algorithm tractable and scalable for high-dimensional and continuous problems. We discuss the gap between the necessary assumptions and the practical situation in \cref{sec.gap}. More details about implementation can be found in \cref{sec.algo_details}. 
% \commenthaitong{Explain bias between assumptions and implementations}
% \commenthaitong{explain projections in practical implementations?}

\section{Experiments}
We aim to answer the following through our experiments:
\begin{itemize}
\setlength{\itemsep}{0pt}
\setlength{\parsep}{0pt}
\setlength{\parskip}{0pt}
\setlength{\partopsep}{-0.1pt}
\setlength{\topsep}{-0.1pt}
    \item Can RCRL learn the largest feasible sets using neural networks approximation of safety value functions? 
    \item Does RCRL outperform CRL methods based on cost value constraints in terms of fewer violations?
    \item Does RCRL perform better than methods based on energy function with respect to performance optimality benefiting from the largest feasible set?
\end{itemize}

\textbf{Benchmarks.} We implement both on- and off-policy RCRL and compare them with different CRL baselines. Experiments include that: (1) use double-integrator \cite{fisac2019hjrl} which has an analytical solution to check the correctness of feasible set learned by RCRL; (2) validate the scalability of RCRL to nonlinear control problems, specifically, a 2D quadrotor trajectory tracking task in safe-control-gym \cite{yuan2021scg}, and (3) classical safe learning benchmark Safety-Gym \cite{Achiam2019safetygym}. Details about each benchmark will be introduced per subsection.

\textbf{Baseline Algorithms.} Details about algorithms can be seen in \cref{sec.algo_details}. Besides RCRL, we test following baselines: (1) \textbf{Lagrangian}-based algorithms whose constraint is about the discounted cumulative costs (an implementation of RCPO \cite{tessler2018rcpo}); (2) \textbf{Reward shaping} method with a fixed coefficient penalty added to the reward; (3) \textbf{CBF}-based algorithms whose constraint is about $B(s)\triangleq\dot{h}(s)+\mu h(s)\le0$ where $\mu\in(0,1)$ is a hyperparameter; and (4) \textbf{SI}-based methods that defines an SI $\varphi(s)=\sigma-(-h(s))^n + k\dot{h}(s)$ and sets constraints 
\begin{equation}
\label{eq.constraint_si}
    \varphi(s')-\max\{\varphi(x)-\eta_D,0\}\le0,
\end{equation}
where $\sigma, n, k, \eta_D$ are hyperparameters.

\subsection{Double Integrator: Comparison to Ground Truth}

We demonstrate that RCRL can learn the largest feasible sets when controlling the double integrator. The reason why we choose the double integrator is that it is a simple dynamical system where we can use numerical solution by the level set toolbox to obtain the ground truth about the largest feasible sets (also called \emph{HJ viability kernels}) \cite{mitchell2007toolbox}. Double integrator is a 2D dynamical system, where the system states and the dynamics are denoted as
\begin{equation}
    s=[x_1,x_2]^T,\quad \dot s = [x_2,a],
\end{equation}
where the control limits of action $a$ is $a\in[-0.5, 0.5]$. The safety constraint is $\|s\|_\infty\leq5$. 
The reward is designed as $r_t=\|s\|^2+a_t^2$.

\textbf{Baselines.} In addition to the ground truth using the level set toolbox \cite{mitchell2007toolbox}, we introduce two discrete approximations of feasible sets using model predictive control (MPC) utilizing CBF constraints and terminal constraints, respectively \cite{ma2021cbf, mayne2000constrained}. These two MPC baselines are named as MPC-CBF and MPC-Terminal, respectively.
\begin{figure}[ht]
    \vskip 0.1in
    \centering
    \includegraphics[width=0.75\columnwidth]{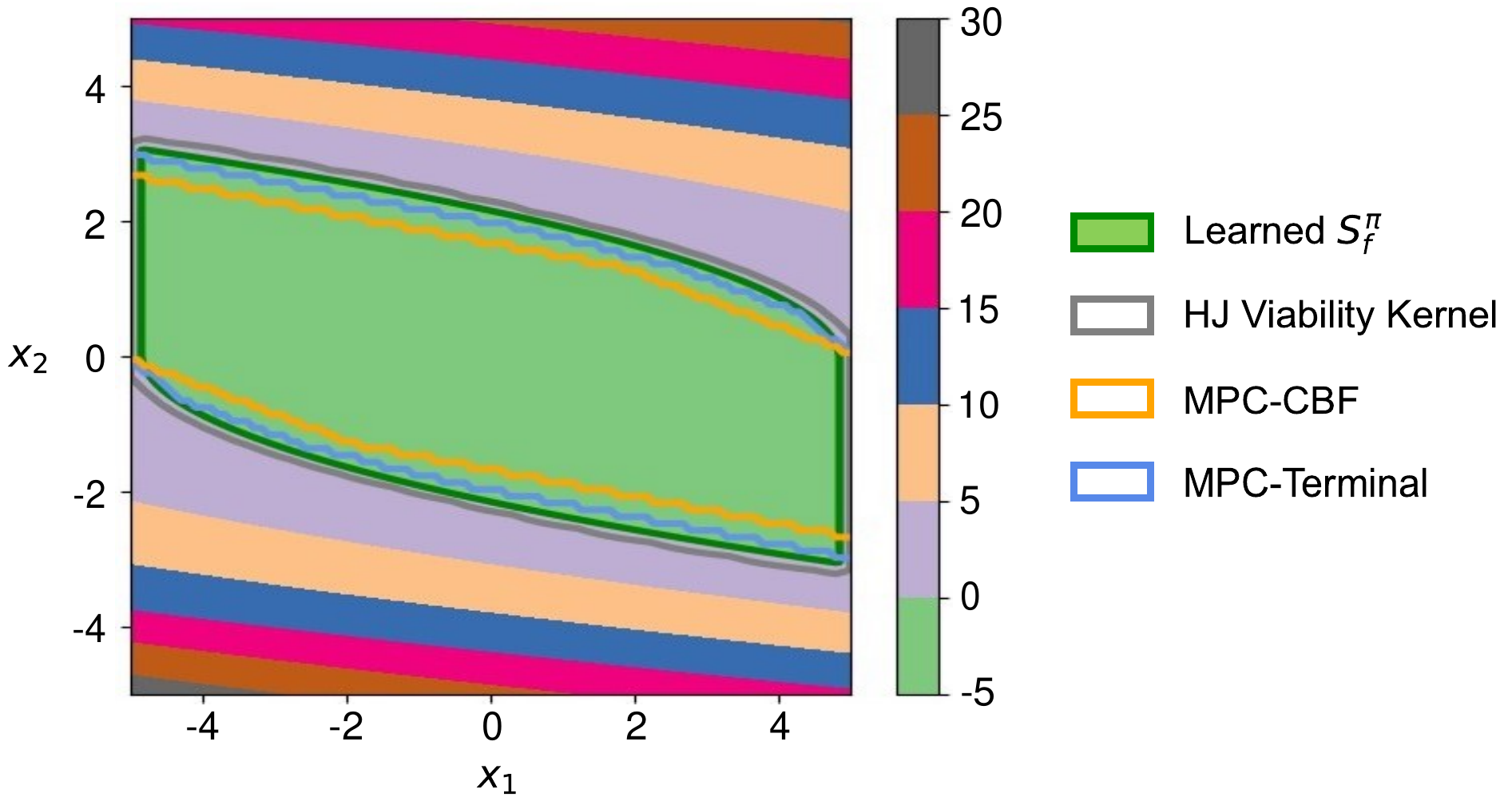}
    \caption{Learned $\mathcal{S}^\pi_f$ of the double integrator.}
    \label{fig:doubleint}
    \vskip -0.2in
\end{figure}

The learned result is shown in \cref{fig:doubleint}. It depicts that the learned $\mathcal{S}_f^\pi$ exactly approximate the largest feasible set or the HJ viability kernel given by the level set toolbox. MPC-Terminal also identifies the feasible sets with some discretization error. MPC-CBF has a smaller feasible set since the conservativeness of energy-based methods.

\subsection{Safe-control-gym: Quadrotor Trajectory Tracking}
The 2D quadrotor trajectory tracking task comes from safe-control-gym \cite{yuan2021scg} and is shown in \cref{fig.snapshot}, where the circle trajectory is marked as red and the constraint for the quadrotor is to keep itself between the two black lines. Schematics of the 2D quadrotor are shown in \cref{fig.quadrotor}, where $(x,z)$ and $(\dot{x},\dot{z})$ are the position and velocity of the COM (center of mass) of the quadrotor in the $xz$-plane, and $\theta$ and $\dot{\theta}$ are the pitch angle and pitch angle rate, respectively. The task for the quadrotor is to track the moving waypoint on the circle trajectory by controlling the normalized thrusts while maintaining its altitude $z$ between $[0.5,1.5]$, i.e.,
\begin{equation}
    \nonumber
    h(s)=\max\{0.5-s^{(2)}, s^{(2)}-1.5\}.
\end{equation}
Details about the state and action space, reward function can be seen in \cref{sec.scg_details}.

\begin{figure}[htb]
% \vskip 0.2in
\begin{center}
    \subfigure[Snapshot of environment, where red line is the reference and black lines are constraint boundaries.]{
    \label{fig.snapshot} 
    \includegraphics[width=0.35\columnwidth]{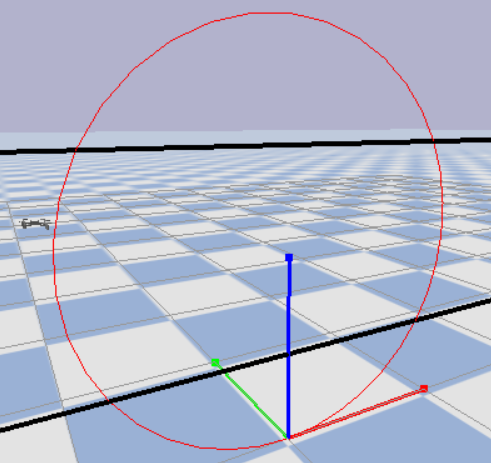}}
\hspace{0.1in}
    \subfigure[Schematics, state and input of the 2D quadrotor]{
    \label{fig.quadrotor} 
    \includegraphics[width=0.4\columnwidth]{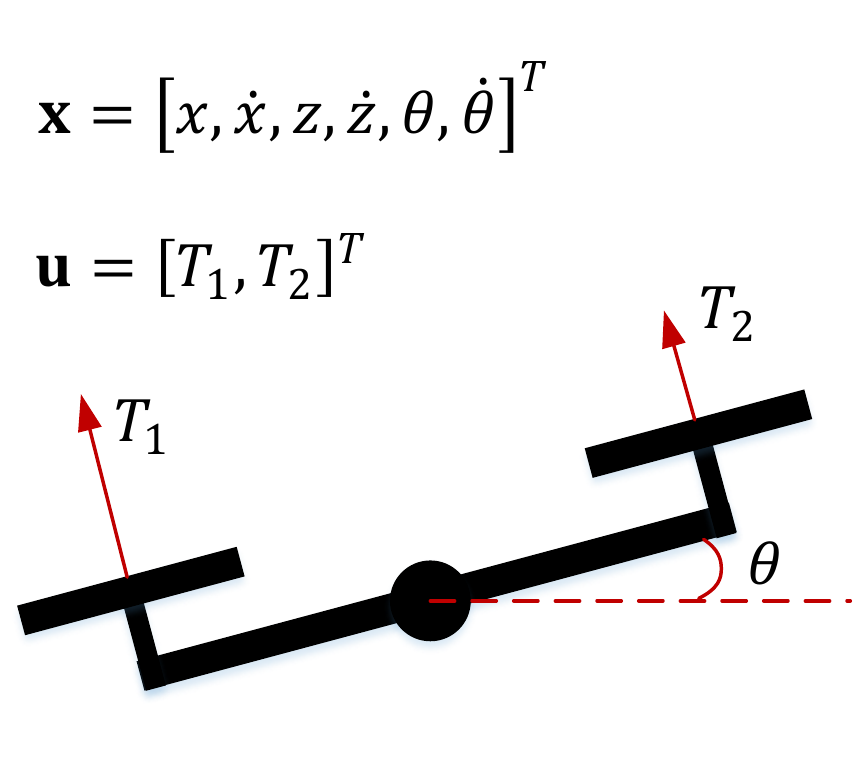}}
\caption{safe-control-gym environment}
\label{fig.scg} 
\end{center}
\vskip -0.1in
\end{figure}

\textbf{Baselines.} We implement an off-policy version of RCRL in safe-control-gym based on SAC \cite{haarnoja2018sac}, forming our Reachable Actor Critic (RAC). Other off-policy baselines are all implemented based on SAC for fairness, including: (1) \textbf{SAC-Lagrangian}, (2) \textbf{SAC-Reward Shaping}, (3) \textbf{SAC-CBF}, and (4) \textbf{SAC-SI}.

\begin{figure}[htb]
% \vskip 0.2in
\begin{center}
    \subfigure{\includegraphics[width=0.48\columnwidth]{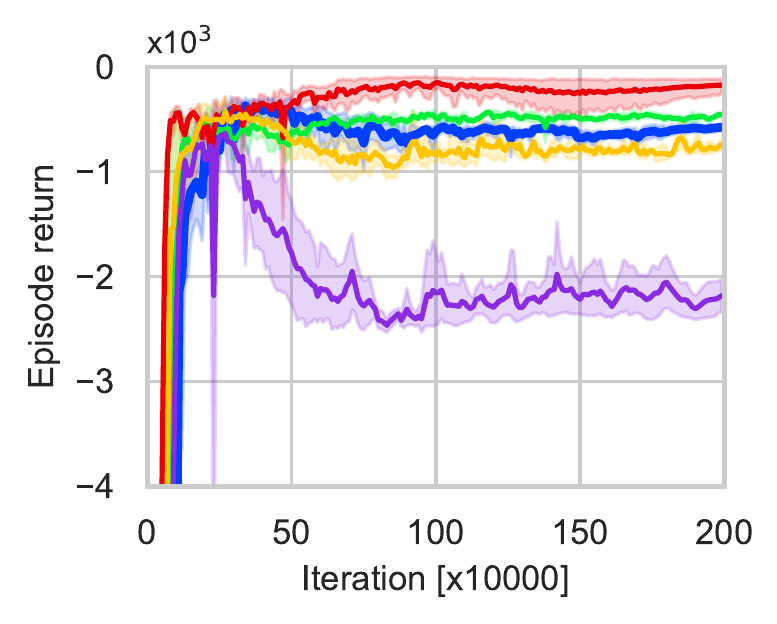}}
\hspace{0.05in}
    \subfigure{\includegraphics[width=0.48\columnwidth]{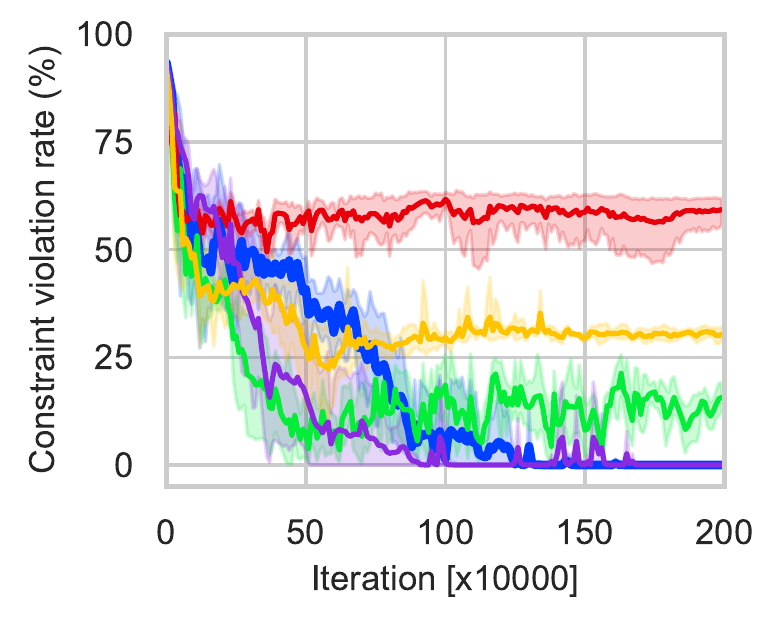}} \\
\vspace{-0.1in}
\subfigure{\includegraphics[width=0.98\columnwidth]{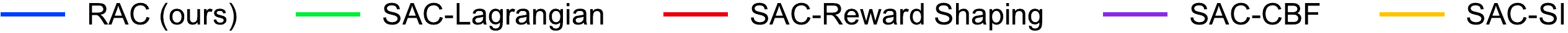}}
% \vspace{-0.1in}
\caption{Performance of algorithms on safety-control-gym. The first two figures are training curves on the quadrotor trajectory tracking task. All results are averaged on 5 independent runs and the shaded regions are the 95\% confidence intervals.}
\label{fig.performance_scg} 
\end{center}
\vskip -0.1in
\end{figure}

\begin{figure}[!htb]
\vskip -3pt
\centering
\includegraphics[width=0.7\columnwidth]{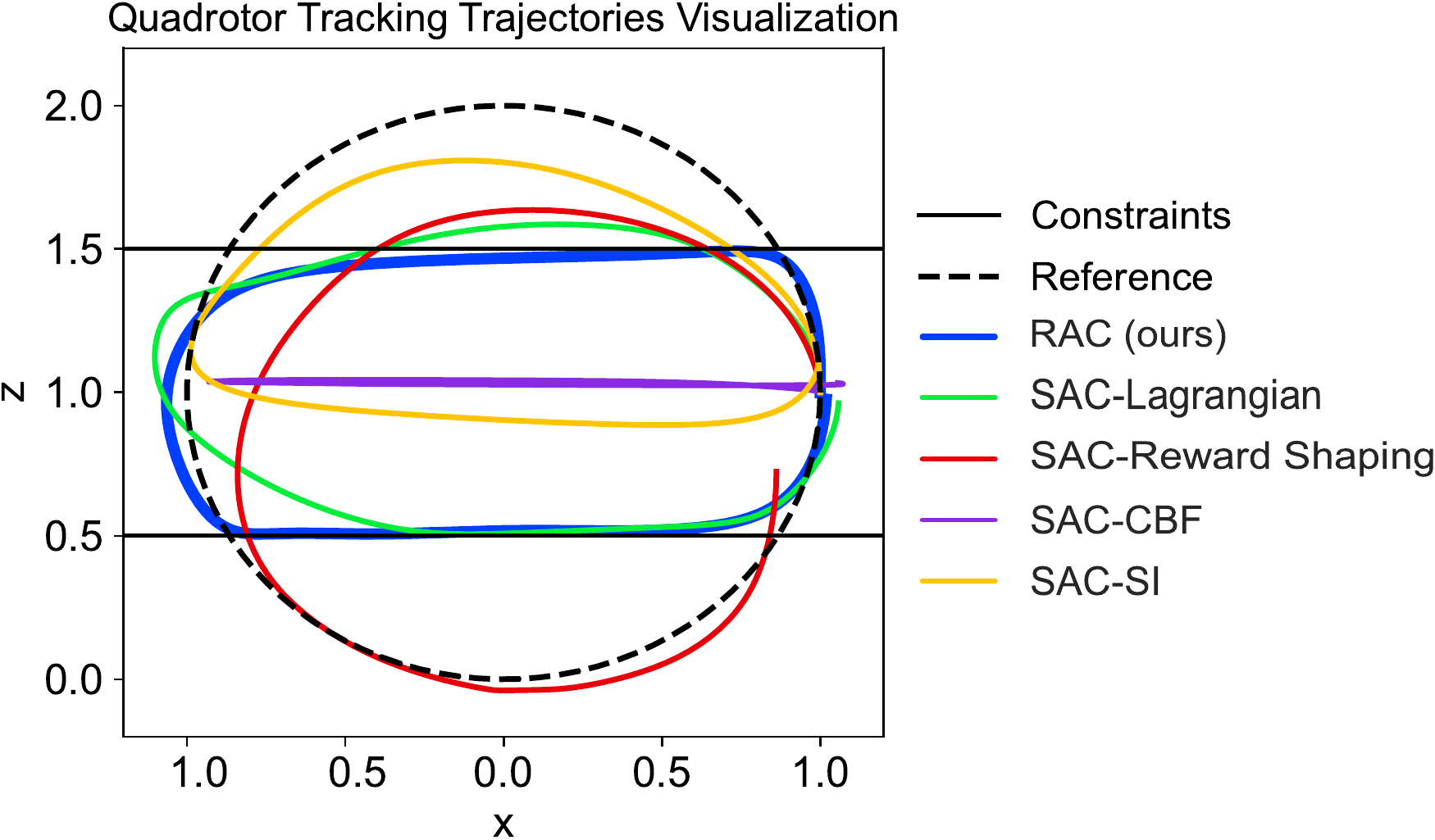}
\vspace{-0.1in}
\caption{Trajectories of final policies trained by different algorithms in a run.}
\label{fig.trajectory} 
\vskip -5pt
\end{figure}

\cref{fig.performance_scg} demonstrates performance with respect to the average return and constraint violation rate of the five algorithms. RAC (blue line) learns a zero-violation policy and reaches near-optimal tracking accuracy. In contrast, though SAC-CBF does not violate the constraint as well, the tracking error is quite large because it just moves horizontally due to a conservative policy. The constraints of SAC-SI require the SI to be below zero and to decrease when it is beyond zero, which explains that it makes the quadrotor fly beyond the upper bound and move horizontally in the safe set, corresponding to the feasible set in \cref{fig.si_fs}. The other two algorithms reach higher returns at the cost of unacceptable constraint violation. In this task, keeping $z$ between $[0.5, 1.5]$ will definitely bring tracking error, which leads to a lower return. Trajectories in \cref{fig.trajectory} indicate that RAC keeps the quadrotor in the safe set strictly and tracks the trajectory accurately inside the safe set while others fail to.

\begin{figure}[!htb]
% \vspace{0.05in}
\begin{center}
\subfigure[RAC safety value function slices]{\label{fig.rac_fs}\includegraphics[width=0.9\columnwidth]{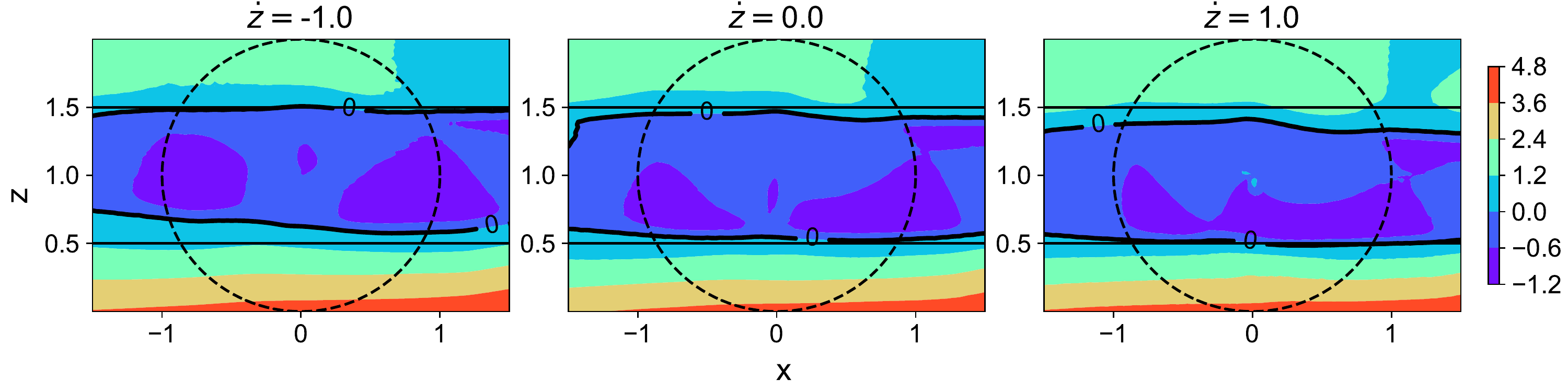}} \\
\vspace{-0.1in}
\subfigure[SAC-CBF control barrier function slices]{\label{fig.cbf_fs}\includegraphics[width=0.9\columnwidth]{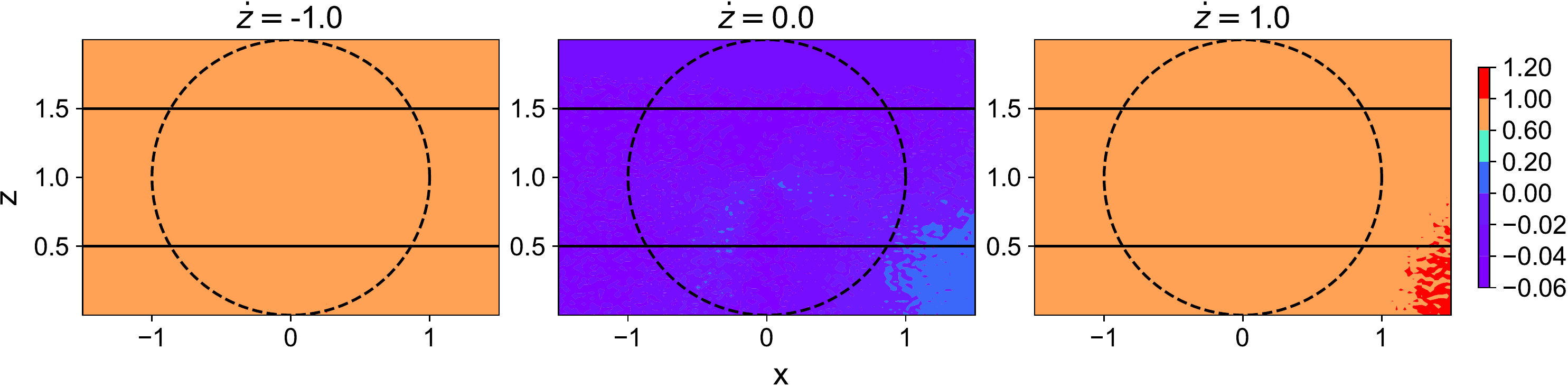}}\\
\vspace{-0.1in}
\subfigure[SAC-SI safety index slices]{\label{fig.si_fs}\includegraphics[width=0.9\columnwidth]{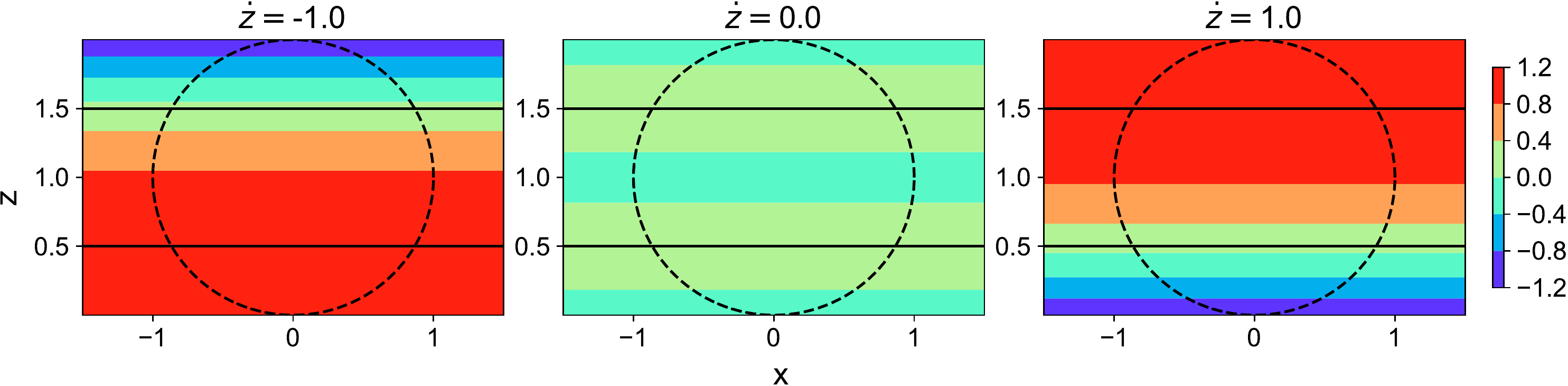}}\\
\vspace{-0.1in}
\caption{The learned feasible set slices on the 2D $xz$-plane with different $\dot{z}$. Values below zero mean that the states are feasible.}
\label{fig.region}
\vskip -0.4in
\end{center}
\end{figure}

\cref{fig.region} shows the slices on $xz$-plane of the feasible sets learned by RAC, SAC-CBF, and SAC-SI with $\dot{z}=-1, 0, 1$. The sub-zero level set of each constrained function represents the learned feasible sets, i.e. $\{s\mid F^\pi(s,a;\phi)\le0\}$ where $F=Q_h, B$, or $\varphi$. In \cref{fig.rac_fs}, the feasible sets of RAC (inside the zero-contour) is smaller than $\SPACE{S}_c$ because when the quadrotor at the boundary of $\SPACE{S}_c$ with a velocity pointing outside, it is doomed to fly out of the space, leading to constraint violation. Thus, such states are supposed to be potentially dangerous and must have a super-zero safety value. Characterizing the feasible set helps RAC track the trajectory accurately inside the safe set (\emph{optimality}) and satisfy the constraint strictly (\emph{safety}). In contrast, an energy function like CBF or SI relies on prior knowledge about the dynamical system. When we choose empirical hyperparameters, the algorithms will possibly learn the wrong feasible set, which leads to either constraint violation or poor performance. As shown in \cref{fig.cbf_fs} and \ref{fig.si_fs}, when $\dot{z}\ne0$, the whole safe set is considered unsafe because of conservativeness. SAC-CBF considers $\dot{z}=0$ as safe, leading to a horizontal-moving policy while SAC-SI leans a wrong feasible set when $\dot{z}=0$, leading to its poor performance.

% $\SPACE{S}_c$ is even infeasible for the quadrotor with nonzero $\dot{z}$. When $\dot{z}=0$, the feasible set of SAC-SI is much smaller than the one learned by RAC while the one of SAC-CBF nearly equals, proving their conservativeness.

\subsection{Safety-Gym: Moving with Sensor Inputs}
All previous tasks have full knowledge of the exact system states in the observations. In this section, we demonstrate the effectiveness of RCRL in complex safe control tasks with only high-dimensional sensor inputs, such as Lidar, on \textbf{Safety-Gym}. \textbf{Safety-Gym} is a widely used safe RL/CRL benchmark. In \cref{fig:safety_gym}, \texttt{point}, \texttt{car} or \texttt{doggo} agents (red) are controlled to reach \texttt{goal} (green) while avoiding \texttt{hazards} (blue, non-physical) or \texttt{pillars} (purple, physical). The tasks are named as \texttt{\{Agent\}}-\texttt{\{Obstacles\}}. The first two tasks have 76D observations space and the last task has 112D observation space and 12D action space.
\begin{figure}[htb]
    \centering
    \subfigure[Car-Hazards]{\includegraphics[width=0.32\linewidth]{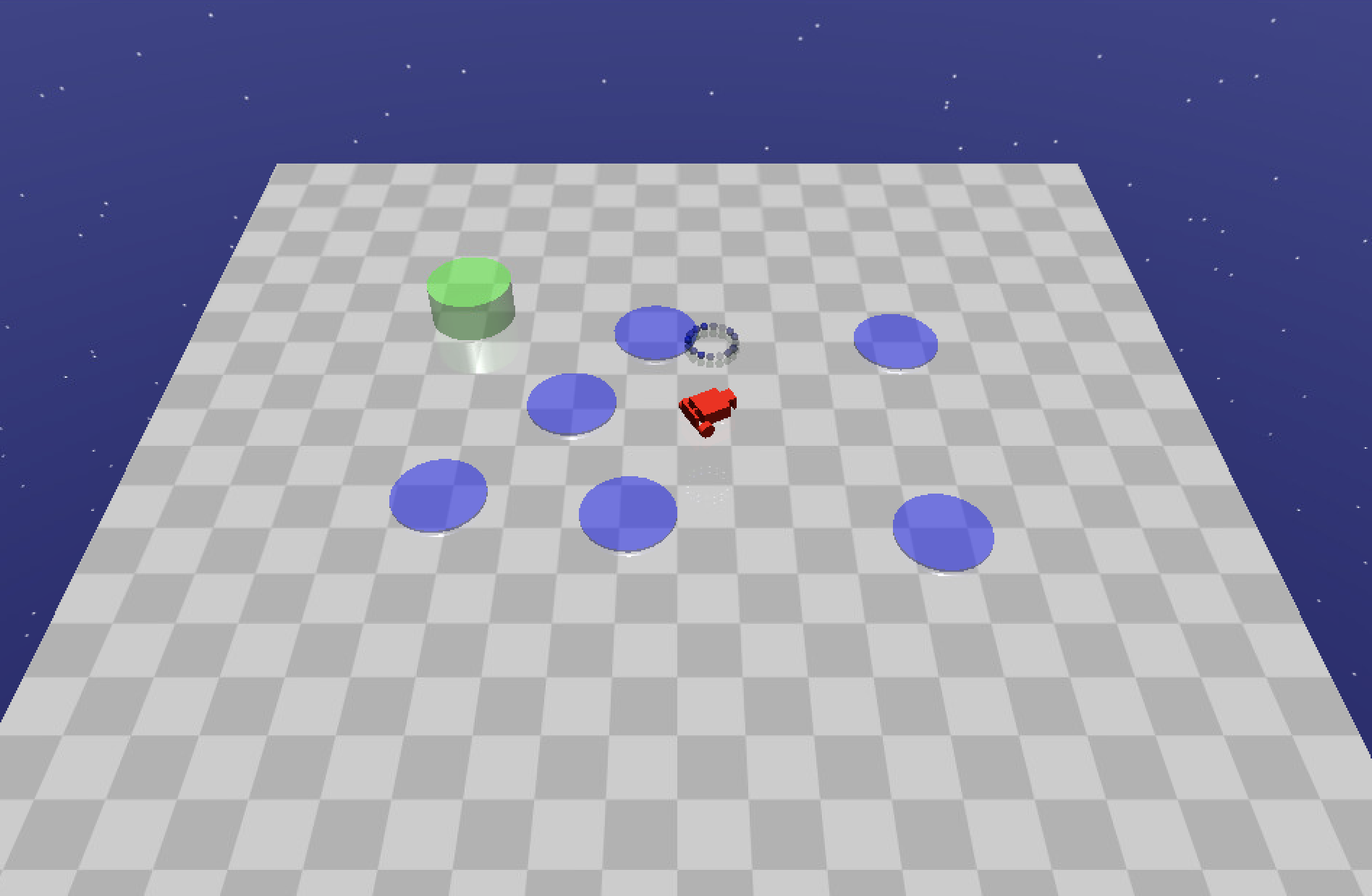}}
    \subfigure[Point-Pillars]{\includegraphics[width=0.32\linewidth]{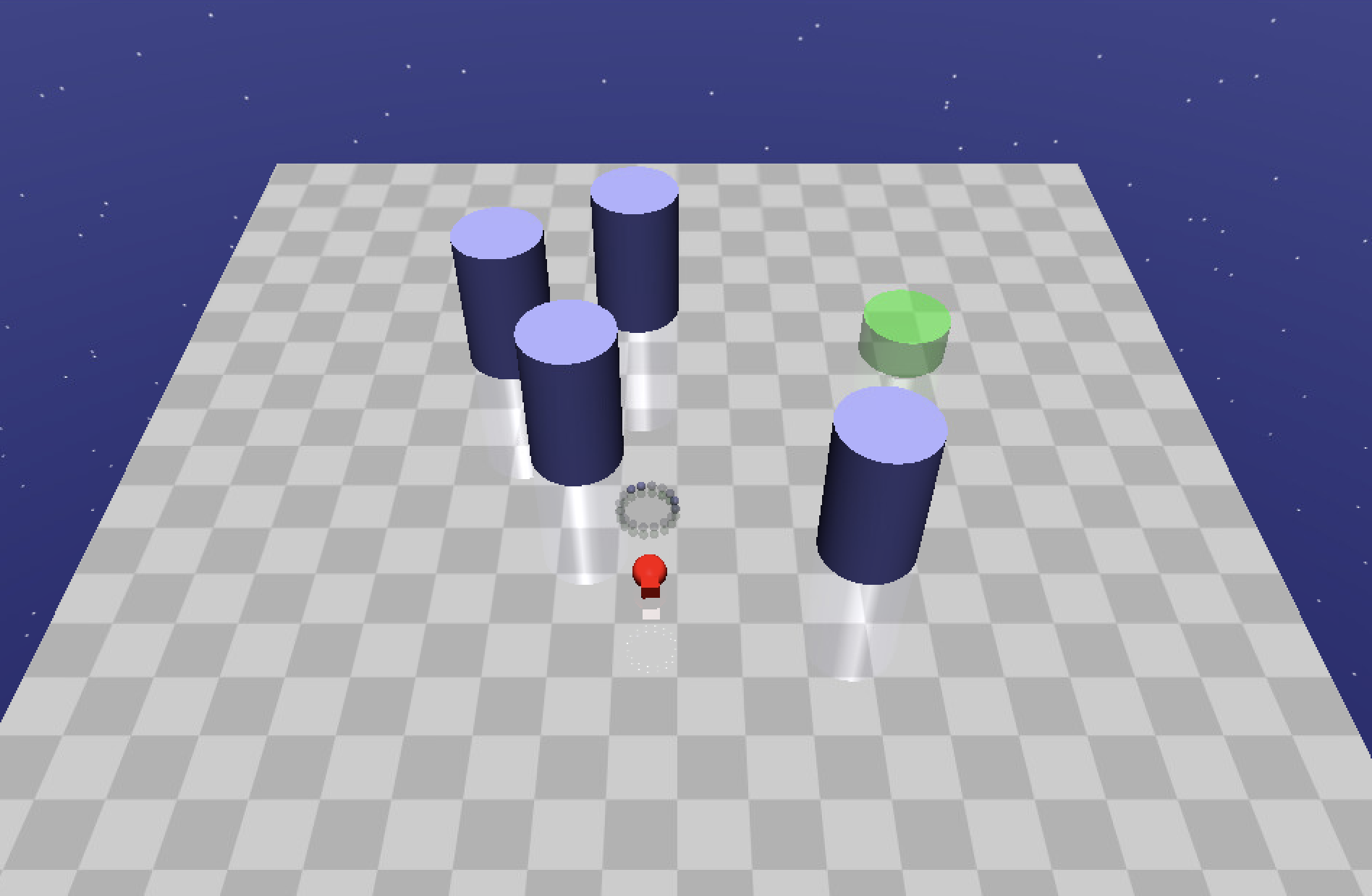}}
    \subfigure[Doggo-Hazards]{\includegraphics[width=0.32\linewidth]{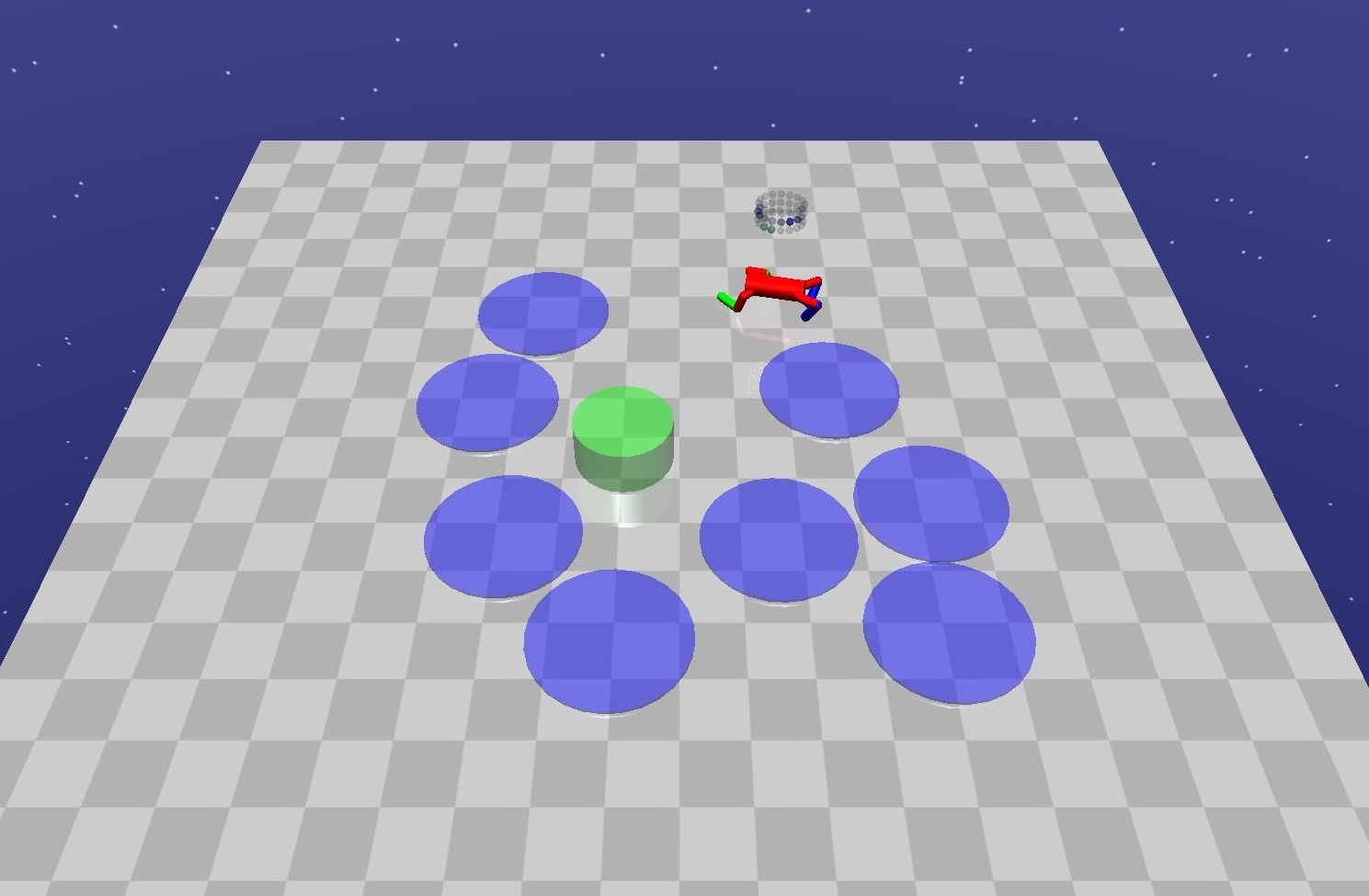}}
    \vspace{-4pt}
    \caption{Safety-Gym tasks.}
    \label{fig:safety_gym}
    \vskip -2pt
\end{figure}
 
\textbf{Baselines.} We name the \emph{on-policy} version of RCRL as Reachability Constrained Optimization \textbf{(RCO)}. Baseline algorithms in Safety-Gym tasks include traditional CRL baseline \textbf{PPO-Lagrangian} \cite{schulman2017ppo, Achiam2019safetygym} and two constrained version of PPO with energy-function based constraints (named as \textbf{PPO-CBF} and \textbf{PPO-SI}), and unconstrained baseline \textbf{PPO}. The distance and relative speed are approximated from the Lidar and speedometer sensors. For multiple obstacles, we select the one with the closest distance (in RCO) or the safety energy decreases (in PPO-SI and PPO-CBF) to compute the safety value functions or energy functions used for constraints.

\begin{figure}[htb]
    \centering
    \includegraphics[width=0.45\linewidth]{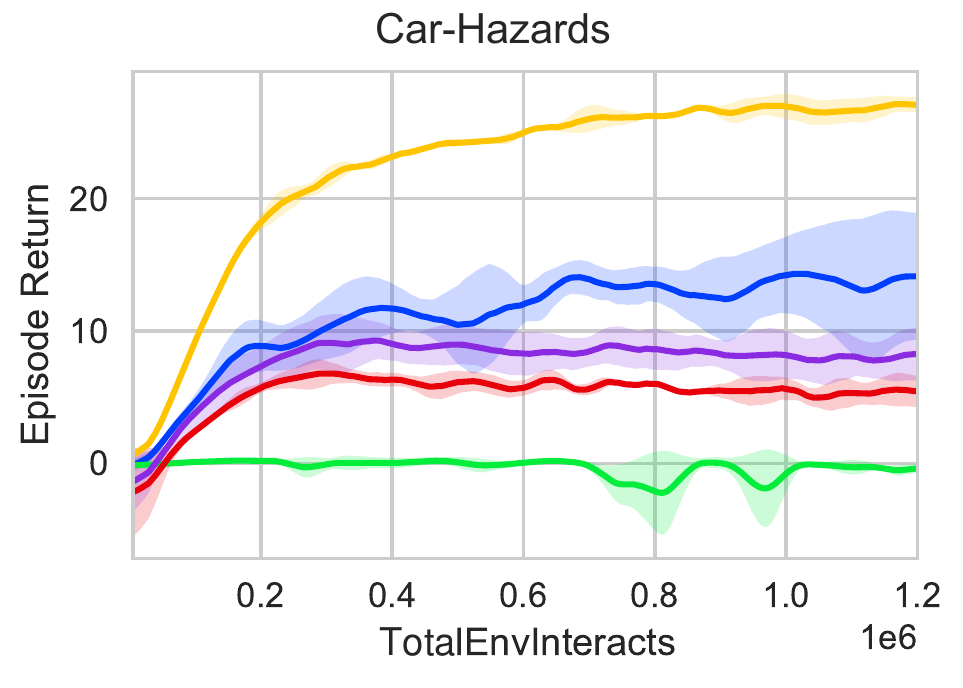}
    \includegraphics[width=0.45\linewidth]{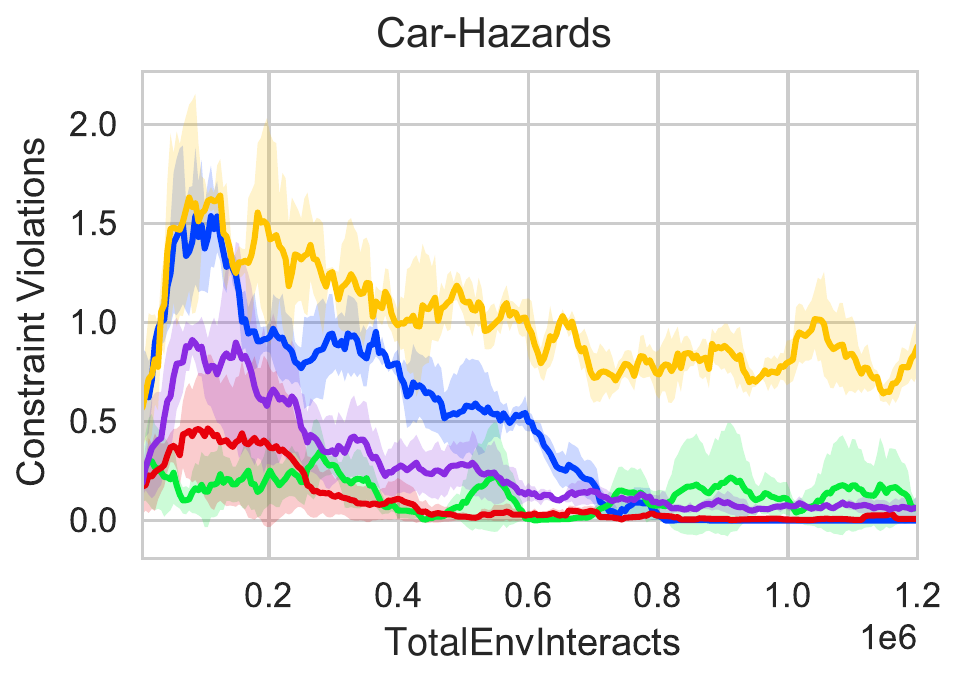}
    \includegraphics[width=0.45\linewidth]{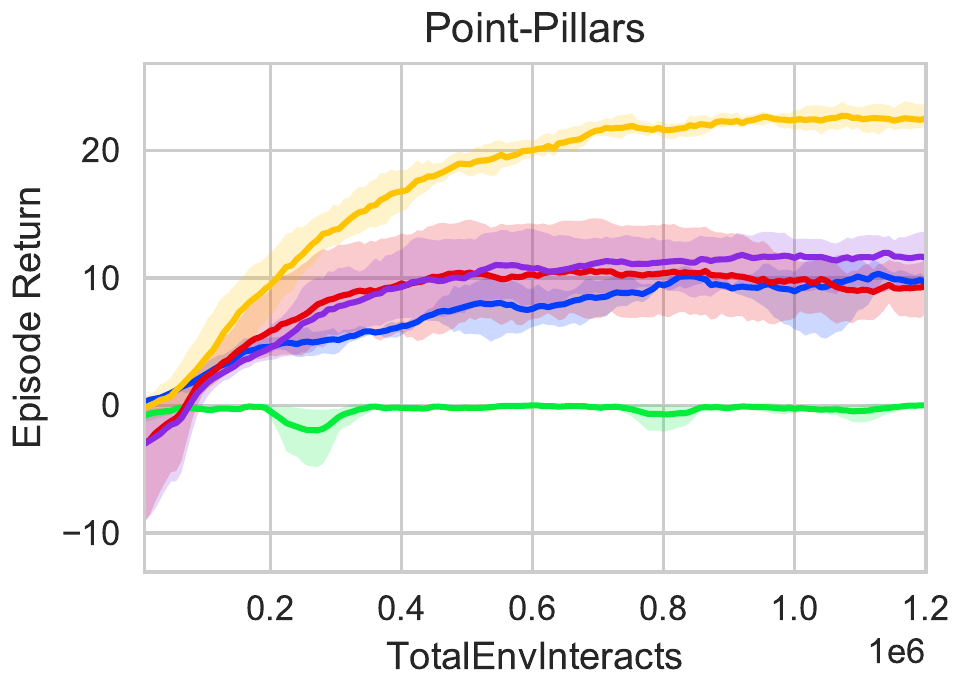}
    \includegraphics[width=0.45\linewidth]{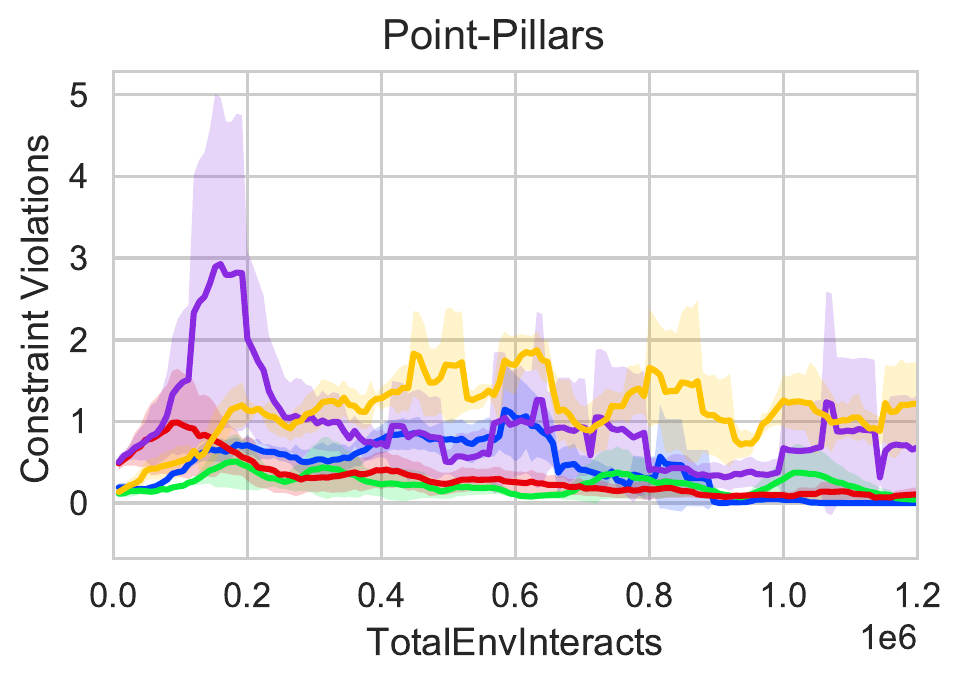}
    \includegraphics[width=0.45\linewidth]{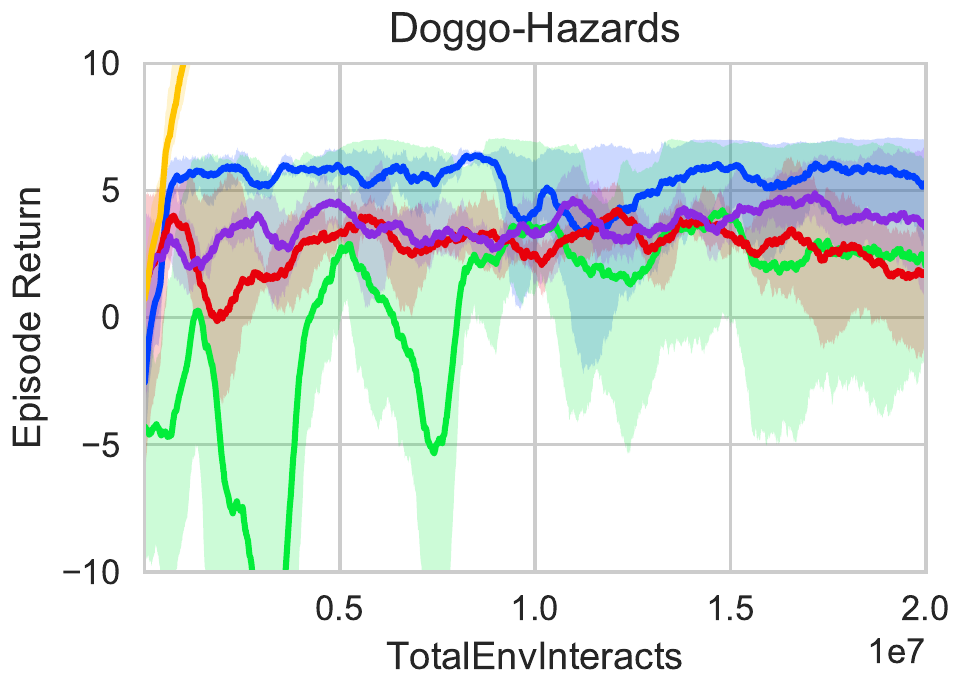}
    \includegraphics[width=0.45\linewidth]{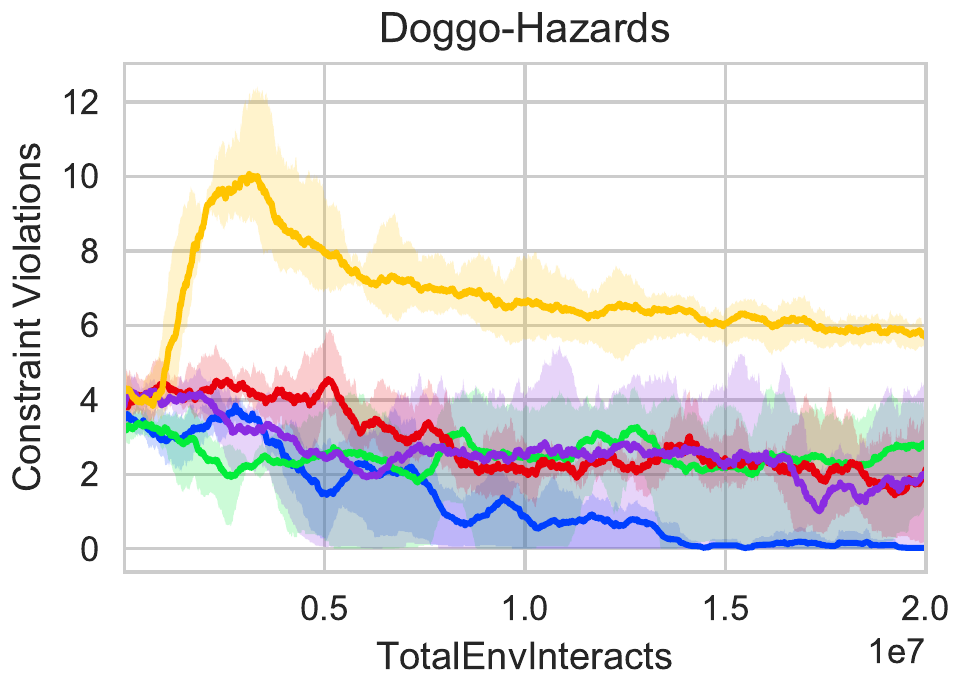}
    \includegraphics[width=0.8\linewidth]{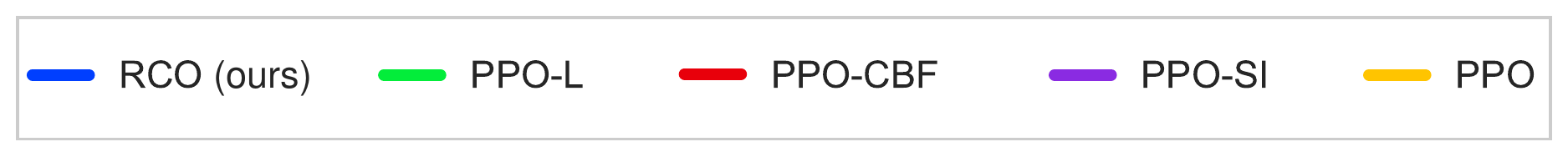}
    \vspace{-0.05in}
    \caption{Training results of Safety-Gym tasks.}
    \label{fig:safety_gym_train}
    \vskip -5pt
\end{figure}

The training results are shown in \cref{fig:safety_gym_train}. The violation rate figures indicate that RCO has the best constraint satisfaction performance, while energy-based baselines PPO-CBF, PPO-SI oscillate around zero violation, indicating RCO learns the feasible sets more exactly. On the contrary, the prior parameters of SI and CBF fail to characterize the exact feasible sets. Meanwhile, RCO has comparable or better return performance, further verifying that RCO reduces the conservativeness compared to energy-based methods.

\section{Concluding Remarks}
We study the novel reachability constraint in CRL, where the safety value function is constrained, guaranteeing persistent constraint-satisfaction. We establish the self-consistency condition of the safety value function, which enables us to characterize the largest feasible set in CRL and consider performance optimality besides safety, avoiding the safety-oriented policies in prior HJ reachability analysis methods. We also prove the convergence of the proposed RCRL with multi-time scale stochastic approximation theory under mild assumptions. Experiments on common benchmarks indicate that RCRL is able to capture the approximate feasible sets, which further guarantees the persistent safety and the competitive performance w.r.t. baselines.

Although empirical results show that RCRL generates persistently safe agents after the convergence of training, like many other Lagrangian-based methods such as \cite{tessler2018rcpo, ma2021fac}, RCRL focuses on safety \textit{after convergence} rather than that \textit{during learning}, while the latter is significant for real-world application of RL algorithms. Furthermore, RCRL explores the boundary of the feasible sets instead of conservatively staying inside them, leading to more violations during the early training stage, which can be seen in \cref{fig.performance_scg} and \ref{fig:safety_gym_train}. We are working on improving RCRL with different approaches such as model-based methods for \textit{safer} exploration and learning.

% Acknowledgements should only appear in the accepted version.
\section*{Acknowledgements}
The authors would like to thank Dr. Jingliang Duan and Wenjun Zou for their valuable suggestions about the problem formulation and the writing of this paper; anonymous reviewers and meta-reviewers for their insightful comments. This work was supported in part by NSF China, under Grant No. U20A201622, U20A20334 and 52072213. This work was also supported by the Ministry of Science and Technology of the People’s Republic of China, the 2030 Innovation Megaprojects ``Program on New Generation Artificial Intelligence" (Grant No. 2021AAA0150000). The authors want to thank support from the Tsinghua University-Toyota Joint Research Center for AI Technology of Automated Vehicle and Horizon Robotics.

% If a paper is accepted, the final camera-ready version can (and
% probably should) include acknowledgements. In this case, please
% place such acknowledgements in an unnumbered section at the
% end of the paper. Typically, this will include thanks to reviewers
% who gave useful comments, to colleagues who contributed to the ideas,
% and to funding agencies and corporate sponsors that provided financial
% support.

% In the unusual situation where you want a paper to appear in the
% references without citing it in the main text, use \nocite
\nocite{Zhang2020Why, duan2021dsac}

\bibliography{RCRL}
\bibliographystyle{icml2022}

%%%%%%%%%%%%%%%%%%%%%%%%%%%%%%%%%%%%%%%%%%%%%%%%%%%%%%%%%%%%%%%%%%%%%%%%%%%%%%%
%%%%%%%%%%%%%%%%%%%%%%%%%%%%%%%%%%%%%%%%%%%%%%%%%%%%%%%%%%%%%%%%%%%%%%%%%%%%%%%
% APPENDIX
%%%%%%%%%%%%%%%%%%%%%%%%%%%%%%%%%%%%%%%%%%%%%%%%%%%%%%%%%%%%%%%%%%%%%%%%%%%%%%%
%%%%%%%%%%%%%%%%%%%%%%%%%%%%%%%%%%%%%%%%%%%%%%%%%%%%%%%%%%%%%%%%%%%%%%%%%%%%%%%
\newpage
\appendix
\onecolumn
% \section{Differences among CRL formulations}
% \label{sec.differences}
% The common formulation of a CRL problem is to maximize the expected cumulative rewards while adhering to the constraints imposed on the expected cost return:
% \begin{equation}
% \tag{CRL}
% \begin{aligned}
% \label{eq.crl_problem}
% \max_{\pi} & \quad \mathbb{E}_{s\sim d_0(s)} [V^{\pi}(s)] \\
% \mbox{subject to} & \quad  \mathbb{E}_{s\sim d_0(s)}[V^{\pi}_c(s)] \le \eta,
% \end{aligned}
% \end{equation}
% where $\eta$ is the cost threshold. However, as stated in Section \ref{sec.related_work}, choosing $\eta$ is tricky and systems are expected to be subject to state constraints at every time step rather than in expectation in safety-critical cases such as collision-avoiding. This can be interpreted as replacing the cost return constraints with \textbf{p}ersistent constraints imposed on each state visited:
% \begin{equation}
% \tag{\textbf{P}CRL}
% \begin{aligned}
% \label{eq.pointwise_problem}
% \max_{\pi} & \quad \mathbb{E}_{s\sim d_0(s)} [V^{\pi}(s)] \\
% \mbox{subject to} & \quad  h(s) \le 0, \forall s\sim d_{\pi}.
% \end{aligned}
% \end{equation}
% However, problem \eqref{eq.pointwise_problem} is hard to solve because the gradient of $d_{\pi}$ w.r.t the policy $\pi$ does not have an analytical form \cite{sutton1999policy}. Therefore, we prove the equivalence between \eqref{eq.pointwise_problem} and \eqref{eq.rcrl_problem} and solve RCRL to get a persistently safe and optimal policy.

\section{Loss Function and Gradients Derivation}
\label{sec.grad_derivation}

The Q-value loss is the mean square error between the approximated Q function and its target \cite{Lillicrap2015ddpg}:
\begin{equation}
\label{eq.q_loss}
\mathcal{J}_{Q}(\omega) = \mathbb{E}_{s\sim \mathcal{D}}[1/2(Q(s,a;\omega)-\hat{Q}(s,a))^2]
\end{equation}
where
\begin{equation}
\nonumber
\hat{Q}(s,a)=r(s,a)+\gamma \mathbb{E}_{s'\sim P}[Q_(s',\pi(s');\omega)].
\end{equation}

Therefore, we are able to derive the stochastic gradients of each objective w.r.t. the parameter vectors in the approximate function. A sample $(s_t, a_t, s_{t+1})$ at time step $t$ is leveraged to compute the stochastic gradients. First, \eqref{eq.q_loss} and \eqref{eq.safety_q_loss} can be optimized with stochastic gradients descent (SGD):
\begin{equation}
\begin{aligned}
\label{eq.q_grad}
\hat{\nabla}_{\omega}{\mathcal{J}_Q}(\omega)=&{\nabla}_{\omega}{Q(s_t,a_t;\omega)} \cdot \left[ Q_\omega(s_t,a_t) - (r(s_t,a_t)+\gamma Q(s_{t+1},a_{t+1};\omega)) \right], \\
\hat{\nabla}_{\phi}{\mathcal{J}_{Q_h}}(\phi)=&{\nabla}_{\phi}{Q_h(s_t,a_t;\phi)} \cdot \left[ Q_h(s_t,a_t;\phi) - ((1-\gamma)h(s)+\gamma \max\{h(s),Q_h(s_{t+1},a_{t+1};\phi)\}) \right],
\end{aligned}
\end{equation}
where $a_{t+1}=\pi(s_{t+1})$.

Combining results from \cite{silver2014dpg, Lillicrap2015ddpg}, the estimated deterministic policy gradient (DPG) is
\begin{equation}
\begin{aligned}
\label{eq.policy_grad}
\hat{\nabla}_{\theta}{\mathcal{J}_\pi}(\theta)=
&{\nabla}_{a}\left[-Q_\omega(s_t,a_t)+\lambda_\xi(s_t)Q_h(s_t,a_t;\phi)\right]|_{a=a_t}\cdot {\nabla_\theta{\pi_\theta(s_t)}}.
\end{aligned}
\end{equation}

Then the stochastic gradient of the multiplier is used to ascend \eqref{eq.policy_loss}:
\begin{equation}
\label{eq.multiplier_grad}
\hat{\nabla}_{\xi}{\mathcal{J}_\lambda}(\xi)={Q_h(s_t,\pi_\theta(s_t);\phi)}\nabla_\xi\lambda_\xi(s_t).
\end{equation}

\section{Proofs}
\subsection{Proof of \cref{theo.consistency}}
\label{proof.sbe}
We only prove the self-consistency condition because the SBE can be proven similarly. From the definition of the safety value function, we know that
\begin{equation}
\begin{aligned}
V_h^\pi(s) &= \max_{t\in \mathbb{N}} h(s_{t}^{\pi}\mid s_0=s) \\
&= \max \{h(s), \max_{t\in\mathbb{N^+}} h(s_t^\pi|s_0=s) \} \\
&= \max \{h(s), \max_{t\in\mathbb{N^+}} h(s_t^\pi|s_1=s') \} \\
&= \max \{h(s), \max_{t\in\mathbb{N}} h(s_t^\pi|s_0=s') \} \\
&= \max \{h(s), V_h^\pi(s') \},
\end{aligned}
\end{equation}
where $s' \sim P(\cdot\mid s,\pi(s))$.

% \subsection{Proof of \cref{theo.equivalent_constraints}}
% \label{proof.equivalent_constraint}
% $\Leftarrow$: We prove it by contradiction. Suppose that there exists $s\sim d(\pi)$ such that $h(s)>0$. Because of $d_{\pi}(s)=(1-\gamma)\sum_{t=0}^\infty \gamma^t \mathbb{P}(s_t=s\mid d_0, \pi)>0$, there exists $t=t_s$ such that $\mathbb{P}(s_{t_s}=s\mid d_0, \pi)>0$, which means there exists a state trajectory $\{s_t^\pi\mid s_{\rm init},\pi\}$ such that $s_{t_s}^\pi=s$, where $s_{\rm init} \sim d_0$. 
% Then according to \cref{def.safety_value_func}
% \begin{equation}
% \nonumber
% \begin{aligned}
% V^\pi_h(s_{\rm init}) &= \max_{t\in \mathbb{N}} h(s_{t}^{\pi}\mid s_0=s_{\rm init}) \\
% &\ge h(s_{t_s}^\pi \mid s_0=s_{\rm init}, \pi) \\
% &= h(s) >0,
% \end{aligned}
% \end{equation}
% which is a contradiction.

% $\Rightarrow$: We also prove it by contradiction. 
% Suppose that there exists $\tilde{s}\sim d_0(s)$ such that $V_h^\pi(\tilde{s})>0$. 
% Then there must exist a time step $\tilde{t}\in\mathbb{N}$ 
% and a state trajectory $\{s_0^\pi, s_1^\pi, \dots, s_{\tilde{t}}^\pi, \dots \mid s_0=\tilde{s}, \pi\}$ such that $h(s_{\tilde{t}}^\pi)>0$. 
% Therefore, $\mathbb{P} ( {s_{\tilde{t}} = s} \mid d_0, \pi ) >0$, 
% meaning $d_\pi(s_{\tilde{t}}^\pi)>0$ and $h(s_{\tilde{t}}^\pi)>0$ hold simultaneously, which brings the contradiction.

\subsection{Proof of \cref{theo.equivalent_region}}
\label{proof.region}
It is obvious that $\SPACE{S}^{\pi}_{f} \subseteq \SPACE{S}_{f}$, we only need to prove that $\SPACE{S}_{f} \subseteq \SPACE{S}^{\pi}_{f}$. When $\SPACE{S}_f\subseteq\SPACE{S}_0$, we have $\SPACE{S}_f \cap \SPACE{S}_0=\SPACE{S}_f$. Therefore, the constraint in \eqref{eq.rcrl_problem} becomes
\begin{equation}
V^{\pi}_h(s) \le 0, \forall s \in \SPACE{S}_f.
\end{equation}
Thus we have $s\in\SPACE{S}_f^{\pi}$ by \cref{def.feasible_set}. In other words, we can conclude that if $s\in \SPACE{S}_f$, we will get $s\in\SPACE{S}_{f}^{\pi}$, which means $\SPACE{S}_{f} \subseteq \SPACE{S}^{\pi}_{f}$.

\subsection{Proof of \cref{theo.equivalent_lag}}
\label{proof.equivalent_lag}
We start from decomposing the surrogate Lagrangian \eqref{eq.surrogate_lag}:
\begin{equation}
\begin{aligned}
& \min_{\pi}\max_{\lambda}\hat{\mathcal{L}}(\pi,\lambda) \\
=& \min_{\pi}\max_{\lambda} \mathbb{E}_{s\sim d_0}[-V^\pi(s) + \lambda(s){V^\pi_h(s)}] \\
=& \min_{\pi}\max_{\lambda} \left\{\mathbb{E}_{s\sim d_0}\left[\left(-V^\pi(s)+\lambda(s){V^\pi_h(s)}\right)\cdot\mathbbm{1}_{s\in\SPACE{S}_f}\right] + \mathbb{E}_{s\sim d_0}\left[\left(-V^\pi(s)+\lambda(s){V^\pi_h(s)}\right)\cdot\mathbbm{1}_{s\notin\SPACE{S}_f}\right] \right\}\\
=& \min_{\pi}\max_{\lambda} \mathbb{E}_{s\sim d_0}\left[\left(-V^\pi(s)+\lambda(s){V^\pi_h(s)}\right)\cdot\mathbbm{1}_{s\in\SPACE{S}_f}\right] \\ &\quad\quad\quad\quad + \min_{\pi} \blue{\max_{\lambda}} \mathbb{E}_{s\sim d_0}\left[(-V^\pi(s)+\lambda(s)\underbrace{V^\pi_h(s)}_{>0})\cdot\mathbbm{1}_{s\notin\SPACE{S}_f}\right]\\
=& \underbrace{\min_{\pi}\max_{\lambda} \mathbb{E}_{s\sim d_0}\left[\left(-V^\pi(s)+\lambda(s){V^\pi_h(s)}\right)\cdot\mathbbm{1}_{s\in\SPACE{S}_f}\right]}_{\rm Part\ 1} \\
&\quad\quad\quad\quad + \underbrace{\min_{\pi} \mathbb{E}_{s\sim d_0}\left[(-V^\pi(s)+\blue{\lambda_{\rm max}}{V^\pi_h(s)})\cdot\mathbbm{1}_{s\notin\SPACE{S}_f}\right]}_{\rm Part\ 2}.
\end{aligned}
\end{equation}
Note that from line 3 to line 4 the $\min\max$ can be decomposed into two parts because the policy (or the multiplier) is statewise and the results $\pi(s)$ (or $\lambda(s)$) of states inside and outside $\SPACE{S}_f$ are independent.

One the one hand, when $\lambda_{\rm max}\to+\infty$, Part 2 will be dominated by $V^\pi_h(s)>0$. Thus, $\lim_{\lambda_{\rm max}\to+\infty}\hat{\pi}^*$ tries to minimize the expected safety value function of initial states outside the largest feasible set. On the other hand, Part 1 is to find the saddle point of the lagrangian of the constrained optimization problem: maximizing the expected return while satisfying the reachability constraint for all initial states inside $\SPACE{S}_f$. In other words,
\begin{equation}
\nonumber
\min_{\pi}\max_{\lambda} \mathbb{E}_{s\sim d_0}\left[\left(-V^\pi(s)+\lambda(s){V^\pi_h(s)}\right)\cdot\mathbbm{1}_{s\in\SPACE{S}_f}\right] \iff
\begin{array}{rl} \max_{\pi} & \mathbb{E}_{s\sim d_0(s)} [V^{\pi}(s) \cdot \mathbbm{1}_{s\in \SPACE{S}_f}] \\ \mbox{subject to} & V^{\pi}_h(s) \le 0, \forall s \in \SPACE{S}_f \cap \SPACE{S}_0 \end{array}.
\end{equation}

Overall, $\lim_{\lambda_{\rm max}\to+\infty}\hat{\pi}^*$ aims to (1) maximize the expected return while satisfy reachability constraints for initial states inside $\SPACE{S}_f$ and (2) minimize the safety value function of initial states outside $\SPACE{S}_f$. This is exactly what problem \eqref{eq.rcrl_problem} does. Therefore, we have $\lim_{\lambda_{\rm max}\to+\infty}\hat{\pi}^*=\pi^*$.

\subsection{Proof of \cref{theo.convergence}}
\label{proof.convergence}
The proof borrows heavily from \cite{chow2017cvar} and \cite{ma2021joint} which both follow the convergence proof of multi-time scale stochastic approximation algorithms in \cite{borkar2009stochastic}. A high level overview of the proof structure is shown as follows.
\begin{enumerate}
    \item By utilizing policy evaluation techniques, we show the critic and safety value function update converge (in the fastest time scale) almost surely to a fixed point solution $(\omega^*, \phi^*)$.
    \item Then, with convergence properties of multi-time scale discrete stochastic approximation algorithms, we show that each update $(\theta_k,\xi_k)$ converges almost surely to a stationary point $(\theta^*,\xi^*)$ of the corresponding continuous time system but with different speeds.
    \item By using Lyapunov analysis, we show that the continuous time system is locally asymptotically stable at the stationary point $(\theta^*,\xi^*)$.
    \item Since the Lyapunov function used in the above analysis is the Lagrangian function $\mathcal{L}(\theta, \xi)$, we conclude that the stationary point $(\theta^*,\xi^*)$ is a local saddle point. Finally by the local saddle point theorem, we deduce that $\theta^*$ is a locally optimal solution for the reachability constrained RL problem.
\end{enumerate}

\textbf{Time scale 1} (Convergence of $\omega$- and $\phi$-updates)
The step size \cref{assume.step_size} tells that $\{\omega_k\}$ and $\{\phi_k\}$ converges on a faster time scale than $\{\theta_k\}$ and $\{\xi_k\}$. According to \citep[Lemma 1, Chapter 6]{borkar2009stochastic}, 
we can treat $(\theta, \xi)$ as arbitrarily fixed quantities during updating $\{\omega_k\}$ and $\{\phi_k\}$. Therefore, we take $(\theta, \xi)=(\theta_k, \xi_k)$, which means the policy and the multiplier are fixed and we are performing policy evaluation to compute $Q^{\pi_{\theta_k}}(s,a)$ and $Q_h^{\pi_{\theta_k}}(s,a)$. With the standard policy evaluation convergence results in \cite{sutton2018rl}, one can easily know that $Q(s,a;\omega_k) \rightarrow Q(s,a;\omega^*) = Q^{\pi_{\theta_k}}(s,a)$ as $k\to\infty$ because the operator $\mathcal{B}$ defined by
\begin{equation}
\nonumber
\mathcal{B}[Q(s,a)]=r(s,a)+\gamma \mathbb{E}_{s'\sim P}[Q(s',\pi(s'))]
\end{equation}
is a $\gamma$ contraction mapping. Thus, all we need to do is to prove that the safety value evaluation in \eqref{eq.q_h_target} is a $\gamma$-contraction mapping as well, which is stated in the following Lemma.
\begin{lemma}[$\gamma$-contraction Mapping]
Under \cref{assume.mdp}, the operator $\mathcal{B}_h$ introduced by $\mathcal{B}_h[Q_h(s,a)]=(1-\gamma)h(s)+\gamma \max\{h(s), \mathbb{E}_{s'\sim P}[Q_h(s',\pi(s'))]\}$ is a $\gamma$-contraction mapping.
\end{lemma}
\textit{Proof.} We study the supremum norm of $\mathcal{B}_h$. For any $Q_h$ and $\hat{Q}_h$, the following holds:
\begin{equation}
\begin{aligned}
\nonumber
\Vert \mathcal{B}_h[Q_h(s,a)] - \mathcal{B}_h[\hat{Q}_h(s,a)] \Vert_\infty &= \gamma \Vert \max\{h(s), \mathbb{E}_{s'\sim P}[Q_h(s',\pi(s'))]\} - \max\{h(s), \mathbb{E}_{s'\sim P}[\hat{Q}_h(s',\pi(s'))]\}\Vert_\infty \\
& \le \gamma\Vert \mathbb{E}_{s'\sim P}[Q_h(s',\pi(s'))] - \mathbb{E}_{s'\sim P}[\hat{Q}_h(s',\pi(s'))]\Vert_\infty \\
& = \gamma\Vert \mathbb{E}_{s'\sim P}[Q_h(s',\pi(s')) - \hat{Q}_h(s',\pi(s'))]\Vert_\infty \\
& \le \gamma \Vert Q_h(s,a) - \hat{Q}_h(s,a)\Vert_\infty.
\end{aligned}
\end{equation} \qed

According to \citep[Proposition A.26]{bertsekas2016nonlinear}, we can conclude that $Q_h(s,a;\phi_k)$ will converge to $Q_h(s,a;\phi^*)=Q_h^{\pi_{\theta_k}}(s,a)$ as $k\to\infty$. Hence, both $\omega_k$ and $\phi_k$ converge to $\omega^*$ and $\phi^*$, respectively and the convergence of \textbf{Time scale 1} is proved.

\textbf{Time scale 2} (Convergence of $\theta$-update) Due to the faster convergence speed of $\theta_k$ than $\xi_k$, we can take $\xi=\xi_k$ when updating $\theta$ according to \citep[Lemma. 1, Chapter 6]{borkar2009stochastic}. Furthermore, since $\omega$ and $\phi$ converge on a faster speed than $\theta$, we have $\Vert Q(s,a;\omega_k)- Q^{\pi_{\theta_k}}(s,a)\Vert\rightarrow0$ and $\Vert Q_h(s,a;\phi_k)- Q_h^{\pi_{\theta_k}}(s,a)\Vert\rightarrow0$ almost surely. Assume that the sample distribution is the same as $\mathcal{D}$. The $\theta$-update from \eqref{eq.policy_grad} is
\begin{equation}
\label{eq.policy_update}
    \theta_{k+1}=\Gamma_\Theta\left[\theta_k-\beta_2(k){\nabla}_{a}(-Q_{\omega_k}(s_t,a)+\lambda_{\xi_k}(s_t)Q_h(s_t,a;\phi_k))|_{a=a_t}\cdot{\nabla_\theta{\pi_\theta(s_t)|_{\theta=\theta_k}}} \right].
\end{equation}
\eqref{eq.policy_update} can also be rewritten as:
\begin{equation}
\nonumber
    \theta_{k+1}=\Gamma_\Theta\left[ \theta_k+\beta_2(k) (-\nabla_\theta \mathcal{L(\theta,\xi)}|_{\theta=\theta_k}+\delta\theta_{k+1}+\delta\theta_\epsilon )\right].
\end{equation}
where
\begin{equation}
\begin{aligned}
\nonumber
    \delta\theta_{k+1} =& \mathbb{E}_{s\sim \mathcal{D}}\left[ \nabla_a(-Q_{\omega_k}(s,a)+\lambda(s;\xi_k)Q_h(s,a;\phi_k))|_{a=\pi(s;\theta_k)} \nabla_\theta\pi(s;\theta)|_{\theta=\theta_k} \right] \\
    &-{\nabla}_{a}(-Q_{\omega_k}(s_t,a)+\lambda_{\xi_k}(s_t)Q_h(s_t,a;\phi_k))|_{a=a_t}\cdot{\nabla_\theta{\pi(s_t;\theta)|_{\theta=\theta_k}}}
\end{aligned}
\end{equation}
and
\begin{equation}
\begin{aligned}
\nonumber
    \delta\theta_\epsilon = \mathbb{E}_{s\sim \mathcal{D}}[ &-\nabla_a(-Q(s,a;\omega_k)+\lambda(s;\xi_k)Q_h(s,a;\phi_k))|_{a=\pi(s;\theta_k)}\nabla_\theta\pi(s;\theta)|_{\theta=\theta_k} \\
    &+ \nabla_a(-Q^{\pi_{\theta_k}}(s,a)+\lambda(s;\xi_k)Q_h^{\pi_{\theta_k}}(s,a))|_{a=\pi(s;\theta_k)} \nabla_\theta\pi(s;\theta)|_{\theta=\theta_k}]
\end{aligned}
\end{equation}

\begin{enumerate}
\item We will show that $\delta\theta_{k+1}$ is square integrable first, specifically,
    \begin{equation}
    \nonumber
    \mathbb{E}[\Vert \delta\theta_{k+1} \Vert^2 \mid \mathcal{F}_{\theta, k}] \le 4\cdot \left[ \Vert \nabla_\theta \pi_\theta(s)|_{\theta=\theta_k} \Vert_\infty^2 \times (\Vert \nabla_a Q(s,a;\omega_k) \Vert_\infty^2 + \Vert \lambda(s;\xi_k) \Vert_\infty^2 \cdot \Vert \nabla_a Q_h(s,a;\phi_k) \Vert_\infty^2)\right].
    \end{equation}
    \cref{assume.differentiability} implies that
    \begin{equation}
    \begin{aligned}
    \nonumber
    \Vert \nabla_\theta \pi_\theta(s)|_{\theta=\theta_k} \Vert_\infty^2 &\le K_1(1+\Vert \theta_k \Vert_\infty^2 ), \\
    \Vert \nabla_a Q(s,a;\omega_k) \Vert_\infty^2 &\le K_2(1+\max_{a\in\SPACE{A}}\Vert a \Vert_\infty^2 ), \\
    \Vert \nabla_a Q_h(s,a;\phi_k) \Vert_\infty^2 &\le K_3(1+\max_{a\in\SPACE{A}}\Vert a \Vert_\infty^2 ).
    \end{aligned}
    \end{equation}
    where $K_1, K_2, K_3$ is three sufficiently large scalars. Furthermore, $\lambda(s;\xi_k)$ can be bounded by $\lambda_{\rm max}$ due to the multiplier upper bound. Because of the aforementioned conditions, we can conclude that $\mathbb{E}[\Vert \delta\theta_{k+1} \Vert^2 \mid \mathcal{F}_{\theta, k}] \le 4 K_1(1+\Vert \theta_k \Vert_\infty^2 )[K_2(1+\max_{a\in\SPACE{A}}\Vert a \Vert_\infty^2 )+\lambda_{\rm max}( K_3(1+\max_{a\in\SPACE{A}}\Vert a \Vert_\infty^2 )]<\infty$. Thus $\delta\theta_{k+1}$ is square integrable.
\item Second, we will show $\delta\theta_\epsilon \rightarrow 0$. Specifically, 
    \begin{equation}
    \begin{aligned}
    \nonumber
    \delta\theta_\epsilon &= \mathbb{E}_{s\sim \mathcal{D}}[ 
    \nabla_a(- Q^{\pi_{\theta_k}}(s,a) + Q(s,a;\omega_k) + \lambda(s;\xi_k)
    (Q^{\pi_{\theta_k}}_h(s,a)-Q_h(s,a;\phi_k))) |_{a=\pi(s;\theta_k)}
    \nabla_\theta\pi(s;\theta)|_{\theta=\theta_k}] \\
    &= \mathbb{E}_{s\sim \mathcal{D}}[ 
    \nabla_a(- Q(s,a;\omega^*) + Q(s,a;\omega_k) + \lambda(s;\xi_k)
    (Q_h(s,a;\phi^*)-Q_h(s,a;\phi_k))) |_{a=\pi(s;\theta_k)}
    \nabla_\theta\pi(s;\theta)|_{\theta=\theta_k}] \\
    &\le \mathbb{E}_{s\sim \mathcal{D}}[\nabla_\theta\pi(s;\theta)|_{\theta=\theta_k}]\cdot(K_4\Vert\omega_k-\omega^*\Vert_\infty+\lambda_{max}K_5\Vert\phi_k-\phi^*\Vert_\infty) \rightarrow 0
    \end{aligned}
    \end{equation}
    where $K_4, K_5$ is the Lipschitz constant. The limit comes from the convergence of the parameters in \textbf{Time scale 1}.
\item Because $\hat{\nabla}_{\theta}{\mathcal{J}_\pi}(\theta)|_{\theta=\theta_k}$ is a sample of $\nabla_\theta \mathcal{L(\theta,\xi)}|_{\theta=\theta_k}$, we can conclude that $\mathbb{E}[\delta\theta_{k+1}\mid \mathcal{F}_{\theta, k}]=0$, where $\mathcal{F}_{\theta, k}=\sigma(\theta_m, \delta\theta_m, m\le k)$ is the filtration generated by different independent trajectories \cite{chow2017cvar}. 
\end{enumerate}
Based on the three facts, the $\theta$-update given by \eqref{eq.policy_update} is a stochastic approximation of the continuous system $\theta(t)$ \cite{borkar2009stochastic}, described by an ODE:
\begin{equation}
\label{eq.ode_theta}
\dot{\theta}=\Upsilon_\theta[-\nabla_\theta\mathcal{L}(\theta, \lambda)],
\end{equation}
where
\begin{equation}
\nonumber
\Upsilon_\theta[F(\theta)] \triangleq \lim\limits_{\eta\rightarrow 0^+}\frac{\Gamma_\Theta(\theta+\eta F(\theta))-\Gamma_\Theta(\theta)}{\eta}
\end{equation}
is the left directional derivative of the function $\Gamma_\Theta(\theta)$ in the direction of $F(\theta)$. The purpose of the directional derivative is to guarantee the update $\Upsilon_\theta[-\nabla_\theta\mathcal{L}(\theta, \lambda)]$ will point in the descent direction along the boundary of $\Theta$ when the $\theta$-update hits the boundary. Invoking the Step 2 in \citep[Appendix A.2]{chow2017cvar}, one can know that $d\mathcal{L}(\theta, \xi)/dt=-\nabla_\theta\mathcal{L}(\theta, \lambda)^T\cdot \Upsilon_\theta[-\nabla_\theta\mathcal{L}(\theta, \lambda)] \le 0$ and the value will be non-zero if $\Vert \Upsilon_\theta[-\nabla_\theta\mathcal{L}(\theta, \lambda)] \Vert \ne 0$. 

Let us consider the continuous system. For a given $\xi$, define a Lyapunov function
\begin{equation}
\nonumber
L_\xi(\theta)=\mathcal{L}(\theta, \xi)-\mathcal{L}(\theta^*, \xi)
\end{equation}
where $\theta^*$ is a local minimum point. Therefore, there exists a scalar $r$ such that $\forall \theta \in B_r(\theta^*)=\{\theta \mid \Vert \theta-\theta^* \Vert \le r \}$, $L_\xi(\theta)\ge0$. Moreover, according to \citep[Proposition 1.1.1]{bertsekas2016nonlinear}, we obtain $\Upsilon_\theta[-\nabla_\theta\mathcal{L}(\theta, \lambda)]|_{\theta=\theta^*}=0$, which means $\theta^*$ is a stationary point.
Due to the non-positive property of $d\mathcal{L}(\theta, \xi)/dt$ and refer to the \citep[Chapter 4]{khalil2002nonlinear}, aforementioned contents show that for any given initial condition $\theta\in B_r(\theta^*)$, the continuous trajectory of $\theta(t)$ of \eqref{eq.ode_theta} converges to $\theta^*$, i.e. $\mathcal{L}(\theta^*, \xi)\le\mathcal{L}(\theta(t), \xi)\le\mathcal{L}(\theta(0), \xi)$ for $\forall t\ge 0$.

Finally, because of the following properties:
\begin{enumerate}
    \item From \citep[Proposition 17]{chow2017cvar}, $\nabla_\theta\mathcal{L}(\theta, \xi)$ is a Lipschitz function in $\theta$;
    \item The step size schedules follow \cref{assume.step_size};
    \item $\delta\theta_{k+1}$ is a square Martingale difference sequence and $\delta\theta\epsilon$ is a vanishing error;
    \item $\theta_k \in \Theta$, which implies that $\sup\limits_{k}\Vert\theta_k\Vert<\infty$ almost surely,
\end{enumerate}
one can invoke \citep[Theorem 2, Chapter 6]{borkar2009stochastic} (multi-time scale stochastic approximation theory) to know that the sequence $\{\theta_k\}, \theta_k\in\Theta$ converges almost surely to the solution of \eqref{eq.ode_theta}, which further converges almost surely to the locally minimum point $\theta^*$.

\textbf{Time scale 3} (Convergence of $\xi$-update) The parameter $\xi$ of multiplier is on the slowest time scale so we can assume that during the $\xi$-update, the policy has converged to the local minimum point, i.e. $\Vert \theta_k - \theta^*(\xi_k)\Vert=0$ and the safety value has converged to a fixed quantity such that $\Vert Q_h(s,a;\phi_k) - Q_h^{\pi_{\theta_k}}(s,a)\Vert=0$. With the continuity of $\nabla_\xi\mathcal{L}(\theta,\xi)$, we have $\Vert \nabla_\xi\mathcal{L}(\theta,\xi)|_{\theta=\theta_k, \xi=\xi_k} - \nabla_\xi\mathcal{L}(\theta,\xi)|_{\theta=\theta^*(\xi_k), \xi=\xi_k} \Vert=0$ almost surely. Thus, the $\xi$-update can be expressed as:
\begin{equation}
\begin{aligned}
\label{eq.xi_update}
\xi_{k+1}&=\Gamma_\Xi\left(\xi_k+\beta_3(k){Q_h(s_t,\pi_{\theta_k}(s_t);\phi_k)}\nabla_\xi\lambda(s_t)|_{\xi=\xi_k}\right) \\
&=\Gamma_\Xi \left( \xi_k+\beta_3(k) ( \nabla_\xi\mathcal{L}(\theta,\xi)|_{\theta=\theta^*(\xi_k),\xi=\xi_k} + \delta\xi_{k+1}) \right)
\end{aligned}
\end{equation}
where
\begin{equation}
\begin{aligned}
\nonumber
\delta\xi_{k+1} &= -\nabla_\xi\mathcal{L}(\theta,\xi)|_{\theta=\theta^*(\xi_k),\xi=\xi_k} + {Q_h(s_t,\pi_{\theta_k}(s_t);\phi_k)}\nabla_\xi\lambda(s_t)|_{\xi=\xi_k} \\
&= - \mathbb{E}_{s\sim \mathcal{D}}[Q_h^{\pi_{\theta^*}}(s,\pi_{\theta^*}(s))\nabla_\xi{\lambda(s;\xi)}|_{\xi=\xi_k} ] + {Q_h(s_t,\pi_{\theta_k}(s_t);\phi_k)}\nabla_\xi\lambda(s_t)|_{\xi=\xi_k} \\
&= - \mathbb{E}_{s\sim \mathcal{D}}[Q_h^{\pi_{\theta^*}}(s,\pi_{\theta^*}(s))\nabla_\xi{\lambda(s;\xi)}|_{\xi=\xi_k} ] + \\ &\qquad \qquad(Q_h(s_t,\pi_{\theta_k}(s_t);\phi_k) - Q_h^{\pi_{\theta_k}}(s_t,\pi_\theta(s_t)) + Q_h^{\pi_{\theta_k}}(s_t,\pi_\theta(s_t)) )\nabla_\xi\lambda(s_t)|_{\xi=\xi_k}.
\end{aligned}
\end{equation}
Similar with $\theta$-update, we need to prove the followings:
\begin{enumerate}
    \item $\delta\xi_{k+1}$ is square integrable because
\begin{equation}
\nonumber
\mathbb{E}[\Vert \delta\xi_{k+1} \Vert^2\mid\mathcal{F}_{\xi,k}] \le 2 \times
\max_{s\in\SPACE{S}} \vert h(s) \vert^2
\times K_6(1+\Vert\xi_k\Vert^2_\infty) < \infty
\end{equation}
where $K_6$ is a sufficiently large number.

\item Since $\Vert Q_h(s_t,\pi_\theta(s_t);\phi_k) - Q_h^{\pi_{\theta_k}}(s_t,\pi_\theta(s_t) \Vert_\infty\rightarrow0$ and ${Q_h^{\pi_{\theta_k}}(s_t,\pi_\theta(s_t))}\nabla_\xi\lambda(s_t)|_{\xi=\xi_k}$ is sample of $\nabla_\xi\mathcal{L}(\theta,\xi)|_{\theta=\theta^*(\xi_k),\xi=\xi_k}$, one can conclude that $\mathbb{E}[\delta\xi_{k+1}\mid\mathcal{F}_{\xi,k}]=0$ almost surely, where $\mathcal{F}_{\xi,k}=\sigma(\xi_m,\delta\xi_m,m\le k)$ is the filtration of $\xi$ generated by different independent trajectories.
\end{enumerate}
Therefore, the $\xi$-update is a stochastic approximation of the continuous system
\begin{equation}
\label{eq.xi_continuous}
\dot{\xi}=\Upsilon_\Xi[\nabla_\xi\mathcal{L}(\theta,\xi)|_{\theta=\theta^*(\xi)}]
\end{equation}
with a Martingale difference error $\delta\xi_k$, where $\Upsilon_\Xi$ is the left directional derivative similarly defined in \textbf{Time scale 2}. Analogous to \textbf{Time scale 2} and in \citep[Appendix B.2]{chow2017cvar}, $d\mathcal{L}(\theta^*(\xi), \xi)/dt={\nabla_\xi\mathcal{L}(\theta,\xi)|_{\theta=\theta^*(\xi)}}^T\cdot\Upsilon_\Xi[\nabla_\xi\mathcal{L}(\theta,\xi)|_{\theta=\theta^*(\xi)}]\ge0$, which is non-zero if $\Vert \Upsilon_\Xi[\nabla_\xi\mathcal{L}(\theta,\xi)|_{\theta=\theta^*(\xi)}] \Vert\ne0$. 

For a local maximum point $\xi^*$, define a Lyapunov function
\begin{equation}
\nonumber
L(\xi)=\mathcal{L}(\theta^*(\xi), \xi^*)-\mathcal{L}(\theta^*(\xi), \xi).
\end{equation}
There exists a scalar $r'$ such that for $\forall \xi \in B_{r'}(\xi^*)=\{\xi\in\Xi\mid\Vert\xi-\xi^*\Vert\le r'\}$, $L(\xi)\ge0$. Moreover, $dL(\xi(t))/dt=-d\mathcal{L}(\theta^*(\xi), \xi)/dt\le0$ and the equal sign only holds when $\Upsilon_\Xi[\nabla_\xi\mathcal{L}(\theta,\xi)|_{\theta=\theta^*(\xi)}]=0$. This means $\xi^*$ is a stationary point. One can invoke \citep[Chapter 4]{khalil2002nonlinear} and conclude that given any initial condition $\xi\in B_{r'}(\xi^*)$, the trajectory of \eqref{eq.xi_continuous} convergences to $\xi^*$, which is a locally maximum point.

Now with (1) $\{\xi_k\}$ is a stochastic approximation to $\xi(t)$ with a Martingale difference error; (2) the step size schedules in \cref{assume.step_size}; (3) the convex and compact property in projection and (4) $\nabla_\xi \mathcal{L}(\theta^*(\xi),\xi)$ is a Lipschitz function in $\xi$. we can apply the multi-time scale stochastic approximation theory again and show that $\{\xi_k\}$ converges to a local maximum point $\xi^*$ almost surely, i.e. $\mathcal{L}(\theta^*(\xi^*), \xi^*) \ge \mathcal{L}(\theta^*(\xi), \xi)$.

\textbf{Local Saddle Point.} From \textbf{Time scale 2} and \textbf{3} we know that $\mathcal{L}(\theta^*(\xi^*), \xi^*) \ge \mathcal{L}(\theta^*(\xi), \xi)$ and $\mathcal{L}(\theta^*, \xi)\le\mathcal{L}(\theta, \xi)$. Thus, $\mathcal{L}(\theta^*, \xi) \le \mathcal{L}(\theta^*, \xi^*) \le \mathcal{L}(\theta, \xi^*)$, which means $(\theta^*, \xi^*)$ is a local saddle point of $\mathcal{L}(\theta, \xi)$. With the saddle point theorem \citep[Proposition 5.1.6]{bertsekas2016nonlinear}, we finally come to the conclusion that $\pi(\cdot ; \theta^*)$ is a locally optimal policy to the RCRL problem \eqref{eq.rcrl_problem}.

\section{Implementation Details of Algorithms}
\label{sec.algo_details}

\subsection{The Gap between Assumptions and Practical Implementations}
\label{sec.gap}
\textbf{Finite MDP.} The boundness of $\SPACE{S}$, $\space{A}$, and reward function can be guaranteed in common RL tasks. However, it is in most of the cases that $\SPACE{S}$ and $\SPACE{A}$ are continuous such that they are infinite. One can discretize the space to get a finite one at the cost of inaccuracy but we will keep the space continuous.

\textbf{Parameterized approximation.} A popular choice of function approximators is deep neural networks (NN) that is differentiable w.r.t. its parameters. However, the general conclusion about the continuity and Lipschitz constant of a NN is still an open problem \cite{kim2021lipschitz}. We still adopt NN in our experiments and leverage clipped gradient update \cite{Zhang2020Why} as the projection operator to keep the parameters of NNs in compact sets as mentioned in \cref{sec.algo}.
% Confirming the Lipschitz constant of a given neural network is still an open problem \cite{kim2021lipschitz}. [for simplification]
% However, empirical results still show the convergence of practical implementations. Furthermore, parameterized statewise Lagrange multipliers remain an unsolved problem and \cite{qin2021dcrl} utilizes a fixed-dimensional vector with interpolating, scalable to the whole state space but sacrificing the representation ability. 
Moreover, a Lagrange multiplier network introduced by \cite{ma2021fac} is used for statewise constraint-satisfaction.

\textbf{Step sizes.} Actually we cannot make any schedule of learning rates to make their sum goes to infinity due to a limited number of steps but the sum of the square of learning rates are finite. Furthermore, we utilize $\beta_1(k)>\beta_2(k)>\beta_3(k),\forall k$ to approximate the relationships among the learning rates.

\textbf{Exploration issue for deterministic policies.} Deterministic policies may lack exploration due to overestimation error \cite{Lillicrap2015ddpg, fujimoto2018td3} but this can be mitigated by off-policy updates with a replay buffer, where the learning and the exploration is treated independently. Hence, we construct a stochastic policy giving means and variances of a multivariate Gaussian distribution but only take the means during evaluation.

\subsection{Off-policy Parts}
Implementation details about off-policy RL algorithms compared in safe-control-gym are covered in this section. For fair comparison, all methods are implemented under the same code base, see \cite{guan2021mpg}. The only differences among them is the constrained function and some hyperparameters, which will be explained in detail in the following content.
\subsubsection{Algorithms}
\begin{table}[!htb]
% \vspace{-0.2in}
\centering
\caption{Off-policy Algorithms Hyperparameters in safe-control-gym}
\label{tab:off_hyper}
\begin{tabular}{@{}ll@{}}
\toprule
Parameter                                  & Value                                    \\ \midrule
\textit{Shared}                            &                                          \\
\quad Optimizer                                  & Adam ($\beta_1=0.99, \beta_2=0.999$)     \\
\quad Approximation function                     & Multi-layer Perceptron                   \\
\quad Number of hidden layers                    & 2                                        \\
\quad Number of neurons in a hidden layer         & 256                                       \\
\quad Nonlinearity of hidden layer               & ELU                                      \\
\quad Nonlinearity of output layer of multiplier               & Softplus                                      \\
\quad Critic/Constrained function learning rate      & Linear annealing 1e-4 $\rightarrow$ 1e-6                     \\
\quad Actor learning rate                        & Linear annealing 2e-5 $\rightarrow$ 1e-6 \\
\quad Temperature coefficient $\alpha$ learning rate & Linear annealing 8e-5 $\rightarrow$ 8e-6                     \\
\quad Reward discount factor ($\gamma$)          & 0.99                                     \\
\quad Policy update interval ($m_{\pi}$)         & 4                                        \\
\quad Multiplier ascent interval ($m_{\lambda}$) & 12                                       \\
\quad Target smoothing coefficient ($\tau$)      & 0.005                                    \\
\quad Max episode length ($N$)                   & 360                                      \\
\quad Expected Entropy ($\bar{\mathcal{H}}$)     & -2                                       \\
\quad Replay buffer size                         & 50,000                                   \\
\quad Replay batch size                          & 512                                      \\ \midrule
\textit{RAC}                               &                                          \\
\quad Multiplier learning rate                   & Linear annealing 6e-7 $\rightarrow$ 1e-7 \\ \midrule
\textit{SAC-Lagrangian}                    &                                          \\
\quad Multiplier learning rate                   & 3e-4                                   \\ \midrule
\textit{SAC-SI}                            &                                          \\
\quad Multiplier learning rate                   & Linear annealing 1e-6 $\rightarrow$ 1e-7 \\
\quad $\sigma, n, k$                             & 0.1, 2, 1                                \\
\quad $\eta_D$                             & 0.1                                \\ \midrule
\textit{SAC-CBF}                           &                                          \\
\quad Multiplier learning rate                       & \multicolumn{1}{c}{Linear annealing 1e-6 $\rightarrow$ 1e-7} \\
\quad $\mu$                                      & 0.1                                  \\ \midrule
\textit{SAC-Reward Shaping}                &                                          \\
\quad Critic learning rate                       & Linear annealing 3e-5 $\rightarrow$ 3e-6                                    \\
\quad Actor learning rate                        & Linear annealing 8e-5 $\rightarrow$ 8e-6                                      \\
\quad Policy update interval ($m_{\pi}$)         & 1                                        \\
\quad $\rho$                                     & 0.5                                      \\
\bottomrule
\end{tabular}
\end{table}

\textbf{RAC} implementation is similar to common off-policy Lagrangian-based CRL methods but with a different constrained function, i.e. the safety value function. As shown in \cref{alg.rcrl}, at each update step gradients of the critic, the safety value function, the actor and the multiplier are computed through samples collected from the environment. The actor is updated on an intermediate frequency and the multiplier at the slowest frequency, correspond to \cref{assume.step_size}.

\textbf{SAC-Lagrangian} is a SAC-based implementation of RCPO \cite{tessler2018rcpo}. The constraint imposed on the RL problem is $\mathcal{J}_c^\pi=\mathbb{E}_{s\sim\mathcal{D},a\sim\pi}[Q^\pi_c(s,a)]\le \eta$, where $Q^\pi_c(s,a)=\sum_t \gamma^t c(s_t|s_0=s,a_0=a,\pi)$ and $\eta$ is the constraint threshold. Bsecause the constraint is about expectation rather than statewise, the multiplier $\lambda$ is a scalar here, but updated with dual ascent similarly with \eqref{eq.policy_loss}.

\textbf{SAC-Reward Shaping} is a SAC-based implementation of fixed penalty optimization (FPO) mentioned in \cite{achiam2017cpo, tessler2018rcpo}. It only adds an additional term in the reward function to punish constrain violation, without any constrained optimization approaches. The reward function during training is modified into $r'(s,a)=r(s,a)-\rho h(s)$, where $\rho>0$ is a fixed penalty coefficient. Then the networks are updated through standard RL and here SAC. Choosing an appropriate $\rho$ is engineering-intuitive and sometimes the tuning process will be time-consuming.

\textbf{SAC-CBF} is inspired by CBF for safe control in control community \cite{dawson2021rclbf, ma2021cbf, choi2021robust}. The core idea is to make potential unsafe behaviors smooth out exponentially as the agent approaches the safe boundary. The constrained function is called barrier function $B(s,a)\triangleq\dot{h}(s)+\mu h(s) \le0$ where $\mu\in(0,1)$ is a hyperparameter.

\textbf{SAC-SI} leverages a human-designed safe index (SI) as the energy function. The control policy needs to keep the system energy low ($\varphi\le0$) and dissipate the energy when the system is at high energy ($\varphi>0$) \cite{ma2021joint}. Hence, the constraint is $\Delta(s,a)\triangleq\varphi(s')-\max\{\varphi(x)-\eta_D,0\}\le0$, where $\eta_D$ is a slack variable controlling the decent rate of SI. A commonly used SI is in the form of $\varphi(s)=\sigma-(-h(s))^n+k\dot{h}(s)$ \cite{zhao2021issa}, which is chosen in this paper.

\subsubsection{Hyperparameters}
\cref{tab:off_hyper} shows the hyperparameters of algorithms evaluated in safe-control-gym.
% Please add the following required packages to your document preamble:
% \usepackage{booktabs}
    % Please add the following required packages to your document preamble:
% \usepackage{booktabs}

\subsection{On-policy Parts}
Implementation details about on-policy RL algorithms including the on-policy version of RCRL benchmarked in Safety-Gym are covered in this section. For fair comparison, all methods are implemented under the same code base, see \cite{Achiam2019safetygym}.
\subsubsection{Algorithms}
\textbf{RCO.} The advantages function for reward value and safety value are denoted as $A^\pi$ and $A^\pi_h$. Denote policy parameterization as $\pi_\theta$, the loss function of RCO when policy parameters, $\theta=\theta_k$ is
\begin{equation}
\mathcal{J}\left(\theta,\xi \right)=\mathbb{E}_{s,a\sim\pi_{\theta_k}}\left\{\overline{A^{\pi_{\theta_{k}}}}(s, a) + \lambda_\xi(s)\overline{A_{h}^{\pi_{\theta_{k}}}}(s, a)\right\}
\label{eq:loss}
\end{equation}
where
$$
\overline{A^{\pi_{\theta_{k}}}}(s,a)=\min \left(\frac{\pi_{\theta}(a \mid s)}{\pi_{\theta_{k}}(a \mid s)} A^{\pi_{\theta_{k}}}(s, a), g\left(\epsilon, A^{\pi_{\theta_{k}}}(s, a)\right)\right),\ g(\epsilon, A)= \begin{cases}(1+\epsilon) A & A \geq 0 \\ (1-\epsilon) A & A<0\end{cases}
$$
$\overline{A^{\pi_{\theta}}_h}(s, a)$ has a similar computation.
\begin{algorithm}[htb]
	\caption{Reachable Constrained Optimization (RCO)}
	\label{alg:rco}
	\begin{algorithmic}[1]
		\REQUIRE Initial policy parameters $\theta_0$, value and cost value function parameters $\omega_0,\phi_0$, multiplier network parameters $\xi_0$
		\FOR{$k=0,1,2,\dots$}
		\STATE Collect set of trajectories $\mathcal D_k=\{\tau_i\}$ with policy $\pi_{\theta_k}$, where $\tau_i$ is a $T$-step episode.
		\STATE Compute reward-to-go $\hat{R}_{t} \doteq \sum\nolimits_{i=t}^{T} \gamma^i r_i$ and cost-to-go $\hat{H_t} \doteq \max_{t} h_t$.
		\STATE Compute advantage functions $A^{\pi_{\theta_{k}}}, A^{\pi_{\theta_{k}}}_h$, according to the value function $V_{\omega_k}$ and safety value function $V_{h_{\phi_k}}$. Compute the multiplier $\lambda_\xi$.
		\STATE Fit value function, safety value function by regression on mean-square error.
		\STATE Update the policy parameters $\theta$ by minimizing \eqref{eq:loss}.
		\STATE Update the multiplier parameters $\xi$ by maximizing \eqref{eq:loss}.
		\ENDFOR
	\end{algorithmic}
\end{algorithm}

\begin{table*}[!htb]
% 	\vskip 0.15in
	\caption{Detailed hyperparameters of on-policy algorithm and baselines.}
	\begin{center}
		\begin{tabular}{lc}
			\toprule
			Algorithm & Value \\
			\hline
			\emph{Shared} & \\
			\quad Optimizer &  Adam ($\beta_{1}=0.9, \beta_{2}=0.999$)\\
			\quad Approximation function  &Multi-layer Perceptron \\
			\quad Number of hidden layers & 2\\
			\quad Number of hidden units per layer & 64\\
			\quad Nonlinearity of hidden layer& ELU\\
			\quad Nonlinearity of output layer (other than multiplier net)& linear\\
			\quad Critic learning rate & Linear annealing $3{\rm{e-}}4\rightarrow 0 $\\
			\quad Reward discount factor ($\gamma$) & 0.99\\
			\quad Cost discount factor ($\gamma_c$) & 0.99\\
			\quad GAE parameters  &  $0.95$\\
			\quad Batch size &  $8000$\\
			\quad Max episode length ($N$) & 1000\\
% 			\midrule
% 			\emph{PPO, PPO-Lagrangian} & \\
			\quad Actor learning rate & Linear annealing $3{\rm{e-}}4\rightarrow 0 $\\
			\quad Clip ratio &  $0.2$\\
			\quad KL margin &  $1.2$\\
% 			\hline
% 			\emph{CPO} &\\ 
% 			\quad Max KL divergence&  $0.1$\\
% 			\quad Damping coefficient&  $0.1$\\
% 			\quad Backtrack coefficient&  $0.8$\\
% 			\quad Backtrack iterations&  $10$\\
% 			\quad Iteration for training values&  $80$\\
% 			\quad Max conjugate gradient iterations& $ 10$ \\
			\midrule
			\emph{RCO, PPO-CBF, PPO-SI} &\\ 
            \quad Nonlinearity of output layer, multiplier net& softplus\\
			\quad Multiplier learning rate & Linear annealing $1{\rm{e-}}4\rightarrow 0 $ \\
			\midrule
			\emph{PPO-Lagrangian} &\\ 
			\quad Init $\lambda$ &  $0.268 (softplus(0))$\\
			\midrule
			\emph{PPO-SI} &\\ 
			\quad $\sigma, n, k$                             & 0.1, 2, 1                                \\
			\midrule
			\emph{PPO-CBF} &\\ 
            \quad $\mu$                   & 0.1 
                       \\
			\bottomrule
		\end{tabular}
	\end{center}
	\vskip -0.1in
	\label{table.hyperonp}
\end{table*}

\textbf{PPO-CBF, PPO-SI.} Constraint functions of these baselines are the same as the off-policy version, only the base algorithm is replaced with PPO \cite{schulman2017ppo}. Compared with \cref{alg:rco}, only the computation of cost-to-go is replaced with the energy-function-based versions.

\subsubsection{Hyperparameters}
See \cref{table.hyperonp}.

\section{Details about Experiments}
\subsection{Quadrotor Trajectory Tracking in safe-control-gym}
\label{sec.scg_details}
Details about the quadrotor trajectory tracking task and training will be covered in this section. The task for the quadrotor is to track a counter-clockwise circle trajectory as accurately as possible while keeping its altitude $z$ between $[0.5,1.5]$, meaning the lower and upper bound of a tunnel. Note that only the next waypoint is accessible to the quadrotor at each time step, so no planning or predictive control in advance exists in this task.

\textbf{Elements of the RL setting.} The state space $\SPACE{S}\subseteq\mathbb{R}^{12}$ consists of the current state of the quadrotor $\VECTOR{x}=[x,\dot{x},z,\dot{z},\theta,\dot{\theta}]^T$ and the information of the next waypoint $\VECTOR{x}^{\rm ref}$, thus $s_t=[\VECTOR{x}_t;\VECTOR{x}^{\rm ref}_t]$. The action is the thrusts given by the two motors on both sides $[T_1,T_2]$, whose value will be normalized to $[0,1]\times[0,1]$. The system dynamics and information about the whole trajectory are inaccessible to the agent. The circle center is at $(0,1)$ and its radius is $1$. The circle is discretized into 360 points so at each time step the reward function is the weighted sum of the difference between $(\VECTOR{x},a)$ and the reference $(\VECTOR{x}^{\rm ref},a^{\rm ref})$, specifically, $r(s_t,a_t)=-(\VECTOR{x}_t-\VECTOR{x}^{\rm ref}_t)^T Q (\VECTOR{x}_t-\VECTOR{x}^{\rm ref}_t) - (a_t-a^{\rm ref}_t)^T R (a_t-a^{\rm ref}_t) $ where $Q=diag(10,1,10,1,0.2,0.2), R=diag(1e-4, 1e-4)$. The constraint is $0.5 \le z \le 1.5$. 

\textbf{Initialization.} For better exploration and generality of the learned feasible set and policy, we initialize the quadrotor uniformly in a rectangle in the $xz$-plane with uniformly distributed vertical and horizontal speed, pitch angle and pitch angle rate, specific ranges in \cref{tab:initial_param}. The nearest discrete waypoint on the trajectory to the initial location of the quadrotor is assigned as the start waypoint. In other words, the start waypoints change as the initial location changes. Quadrotor initialized at the center will be assigned a start waypoint randomly.

\begin{table}[htb]
\vspace{-0.1in}
\centering
\caption{The Initialization Range of Each Variable}
\label{tab:initial_param}
\begin{tabular}{@{}cl@{}}
\toprule
Variable        & \multicolumn{1}{c}{Range} \\ \midrule
$x$             & $[-1.5,+1.5]$             \\
$\dot{x}$       & $[-1.0,+1.0]$             \\
$z$             & $[0.25,+1.75]$            \\
$\dot{z}$       & $[-1.5,+1.5]$             \\
$\theta$        & $[-0.2,+0.2]$             \\
$\dot{\theta}$ & $[-0.1,+0.1]$             \\ \bottomrule
\end{tabular}
\end{table}

\textbf{Training.} At each time step, the quadrotor outputs the two torques based on its state, including the waypoint next to the one in the last time step in the counter-clockwise direction. Then it receives the state transition, the reward and constraint function or cost signal. The $(s,a,r,s',h,c)$ will be sent to the replay buffer. Simultaneously, the learner gets batches of samples from the replay buffer and compute gradients to update the function approximators. The maximum length $T$ of an episode equals to the number of the discrete waypoints, i.e. $360$. The episode will be ended and reset when the maximum length is reached or the quadrotor flies out of the bounded region $\{s \mid |x|\le 2, |z|\le3 \}$.

\textbf{Evaluation.} The policy is evaluated for four runs at one time. It is initialized statically at $(1,1),(-1,1),(0,0.53),(0,1.47)$ respectively in a run where safety can be guaranteed by hovering so the four initial states are feasible. Then the average return ${\sum_{t=0}^{T-1} r(s_t,a_t)}$ and constraint violation rate $\frac{\sum_{t=0}^{T-1} c(s_t)}{T}$ are taken as the performance and constraint-satisfaction metrics, respectively.

\textbf{Feasible sets slices.} The approximated constrained function in each algorithm ($Q_h^\pi(s,a), B^\pi(s,a), \Delta(s,a)$ in RAC, SAC-CBF, SAC-SI, respectively) is a function $f: \mathbb{R}^{12}\mapsto\mathbb{R}$. Hence, we need to project the high-dimensional state to a lower one to visualize the constrained function. Because the imposed constraints is about the $z$-coordinate, we choose to project each state onto the $xz$-plane and observe the changing trend with varying $\dot{z}$. Coordinates in $x$- and $z$-axis are uniformly sampled from the set $\{(x,z)\mid|x|<1.5, 0.5<z<1.5\}$ while $\dot{z}$ is chosen among $\{-1,0,1\}$ and $\dot{x}$, $\theta$ and $\dot{\theta}$ are all set to zero. The tracking waypoint of each sample is the nearest one on the circle trajectory to the $(x,z)$ sample. Then we generate state $s$ for a given $(x,z)$ tuple according to the aforementioned rules and get action $a$ from the trained policy. The constrained value can be calculated with $f(s,a)$.

% \subsection{Safety-Gym}
%%%%%%%%%%%%%%%%%%%%%%%%%%%%%%%%%%%%%%%%%%%%%%%%%%%%%%%%%%%%%%%%%%%%%%%%%%%%%%%
%%%%%%%%%%%%%%%%%%%%%%%%%%%%%%%%%%%%%%%%%%%%%%%%%%%%%%%%%%%%%%%%%%%%%%%%%%%%%%%

\end{document}